\documentclass{article}

\newif\ifarxiv
\arxivtrue

\ifarxiv
\else
\PassOptionsToPackage{numbers}{natbib}
\fi

\usepackage[final]{neurips_2019}
\usepackage[utf8]{inputenc} 
\usepackage[T1]{fontenc}    
\usepackage{url}            
\usepackage{booktabs}       
\usepackage{amsfonts}       
\usepackage{nicefrac}       
\usepackage{microtype}      

\usepackage{silence}
\usepackage{amsmath,amsfonts,amssymb,amsthm,thmtools,dsfont}
\usepackage{mathtools}
\usepackage{xparse}
\usepackage{enumitem} 
\usepackage{etoolbox}
\usepackage{amsfonts}
\usepackage{hhline}
\usepackage{multirow}
\usepackage{tabularx}
\usepackage{chngpage} 

\usepackage[all]{foreign}   

\usepackage{hyperref}

\usepackage[capitalise,nameinlink,noabbrev]{cleveref}

\declaretheoremstyle[bodyfont=\itshape,notefont=\bfseries]{thmbf}
\declaretheoremstyle[notefont=\bfseries]{defibf}
\declaretheoremstyle[headfont=\normalfont\itshape,qed=$\square$]{proofita}

\declaretheorem[style=thmbf,numberwithin=section,name={Theorem}]{theorem}
\declaretheorem[style=thmbf,numberlike=theorem,name={Proposition}]{proposition}

\declaretheorem[style=thmbf,numberlike=theorem,name={Algorithm}]{algorithmm}
\declaretheorem[style=thmbf,numberlike=theorem,name={Corollary}]{corollary}
\declaretheorem[style=thmbf,numbered=no,name={Corollary}]{corollary*}

\declaretheorem[style=defibf,numberlike=theorem,name={Example}]{example}
\declaretheorem[style=defibf,numberlike=theorem,name={Definition}]{definition}

\declaretheorem[style=defibf,numbered=no,name={Remark}]{remark}

\let\epsilon\varepsilon
\let\subset\subseteq

\DeclarePairedDelimiter\abs{\lvert}{\rvert}     
\DeclarePairedDelimiter\pa{\lparen}{\rparen}    
\DeclarePairedDelimiter\norm{\lVert}{\rVert}    

\NewDocumentCommand{\infnorm}{ s O{} m }{%
  \IfBooleanTF{#1}{\norm*{#3}}{\norm[#2]{#3}}_{\infty}%
}
\NewDocumentCommand{\twonorm}{ s O{} m }{%
  \IfBooleanTF{#1}{\norm*{#3}}{\norm[#2]{#3}}_2%
}
\NewDocumentCommand{\tvnorm}{ s O{} m }{%
  \IfBooleanTF{#1}{\norm*{#3}}{\norm[#2]{#3}}_{\textup{TV}}%
}
\NewDocumentCommand{\onenorm}{ s O{} m }{%
  \IfBooleanTF{#1}{\norm*{#3}}{\norm[#2]{#3}}_1%
}
\NewDocumentCommand{\frobnorm}{ s O{} m }{%
  \IfBooleanTF{#1}{\norm*{#3}}{\norm[#2]{#3}}_F%
}
\NewDocumentCommand{\scalar}{s O{} >{\SplitArgument{1}{,}}m}{%
    \IfBooleanTF{#1}{\scalaraux*#3}{\scalaraux[#2]#3}%
}
\DeclarePairedDelimiterX{\scalaraux}[2]{\langle}{\rangle}{#1, #2}

\newcommand*\circledaux[1]{\tikz[baseline=(char.base)]{
    \node[shape=circle,draw,inner sep=0.8pt] (char) {#1};}}

\NewDocumentCommand{\circled}{ m o }{%
    \IfNoValueTF{#2}{ \circledaux{#1} }{ \stackrel{\circledaux{#1}}{#2} }%
}

\newcommand*\nn{\textsc{nn}}
\newcommand*\rnn{\textsc{rnn}}

\newcommand*\lstm{\textsc{lstm}}
\newcommand*\gru{\textsc{gru}}
\newcommand*\qr{\textsc{qr}}
\newcommand*\svd{\textsc{svd}}
\newcommand*\rgd{\textsc{rgd}}
\newcommand*\sgd{\textsc{sgd}}
\newcommand*\vae{\textsc{vae}}
\newcommand*\adam{\textsc{adam}}
\newcommand*\adagrad{\textsc{adagrad}}
\newcommand*\rmsprop{\textsc{rmsprop}}
\newcommand*\scornn{\textsc{scornn}}
\newcommand*\scurnn{\textsc{scurnn}}

\newcommand*\exprnn{\textsc{exprnn}}
\newcommand*\dtriv{\textsc{dtriv}}
\newcommand*\mnist{\textsc{mnist}}
\newcommand*\pmnist{\textsc{p-mnist}}
\newcommand*\timit{\textsc{timit}}
\newcommand*\mse{\textsc{mse}}

\DeclareMathOperator*{\cay}{cay}                        
\renewcommand*{\Im}{\mathrm{Im}}                        
\newcommand*\dif{\mathrm{d}}                            
\newcommand*\Id{\mathrm{Id}}                            
\newcommand*\grad{\nabla}                               
\newcommand\conn{\nabla}                                
\newcommand*\defi{\coloneqq}                            
\newcommand*\iso{\cong}                                 
\newcommand*\fn{\operatorname}                          
\newcommand*\Unif{\mathcal{U}}                          
\newcommand{\code}{\texttt}                             
\newcommand*\gm{\textsl{g}}                             

\newcommand*\CC{\mathbb{C}}                              
\newcommand*\RR{\mathbb{R}}                              

\let\epsilon\varepsilon
\let\subset\subseteq

\newcommand*\deffun[1]{\dodeffunction#1\relax}
\def\dodeffunction#1:#2->#3;#4\relax
    {\ifblank{#4}
    {#1\colon#2\to#3}
    {\!\begin{aligned}#1\colon#2&\to#3\\ \dodeffunctionaux#4\relax\end{aligned}}}
\def\dodeffunctionaux#1->#2\relax{#1&\mapsto#2}

\catcode`\|=\active
\DeclarePairedDelimiterX\set[1]{\lbrace}{\rbrace}
  {\mathcode`\|="8000 \def|{\:\delimsize\vert\:}#1}
\catcode`\|=12 %

\newcommand\transaux{\intercal}                             
\newcommand\trans[1]{#1^\transaux}                          
\newcommand\conj[1]{#1^\ast}                                
\DeclareMathOperator\tr{tr}                                 
\newcommand\dual[1]{#1^\ast}                                
\newcommand\MM{\mathcal{M}}                                 
\newcommand\NN{\mathcal{N}}                                 
\DeclareMathOperator\ad{ad}                                 
\DeclareMathOperator\Ad{Ad}                                 
\newcommand\g{\mathfrak{g}}                                 
\newcommand\I{\mathrm{I}}                                   

\NewDocumentCommand{\enorm}{ s O{} m }{%
    \IfBooleanTF{#1}{\norm*{#3}}{\norm[#2]{#3}}_{\E}%
}
\NewDocumentCommand{\denorm}{ s O{} m }{%
    \dual{\IfBooleanTF{#1}{\enorm*{#3}}{\enorm[#2]{#3}}}%
}

\DeclareMathOperator{\Endaux}{End}                       
\DeclareMathOperator{\Uaux}{U}                           
\DeclareMathOperator{\ualgaux}{\mathfrak{u}}                
\DeclareMathOperator{\SLaux}{SL}                         
\DeclareMathOperator{\slaux}{\mathfrak{sl}}              
\DeclareMathOperator{\GLaux}{GL}                         
\DeclareMathOperator{\GLpaux}{GL^{+}}                    
\DeclareMathOperator{\RPaux}{\mathbb{RP}}                
\DeclareMathOperator{\CPaux}{\mathbb{CP}}                
\DeclareMathOperator{\glaux}{\mathfrak{gl}}              
\DeclareMathOperator{\Oaux}{O}                           
\DeclareMathOperator{\SOaux}{SO}                         
\DeclareMathOperator{\soaux}{\mathfrak{so}}              
\DeclareMathOperator{\symaux}{\mathfrak{sym}}            
\DeclareMathOperator{\SUaux}{SU}                         
\DeclareMathOperator{\suaux}{\mathfrak{su}}              
\DeclareMathOperator{\Spaux}{Sp}                         
\DeclareMathOperator{\Symaux}{Sym}                       
\DeclareMathOperator{\Skewaux}{Skew}                     
\DeclareMathOperator{\Staux}{St}                         
\DeclareMathOperator{\Haux}{\mathbb{H}}                  
\DeclareMathOperator{\Diagaux}{Diag}                     
\DeclareMathOperator{\toraux}{\mathbb{T}}                
\DeclareMathOperator{\toralg}{\mathfrak{t}}              
\NewDocumentCommand{\U}{ m }{ \Uaux\pa{#1} }
\NewDocumentCommand{\ualg}{ m }{ \ualgaux\pa{#1} }
\NewDocumentCommand{\SL}{ m }{ \SLaux\pa{#1} }
\NewDocumentCommand{\slalg}{ m }{ \slaux\pa{#1} }
\NewDocumentCommand{\End}{ m }{ \Endaux\pa{#1} }
\NewDocumentCommand{\GL}{ m }{ \GLaux\pa{#1} }
\NewDocumentCommand{\GLp}{ m }{ \GLpaux\pa{#1} }
\NewDocumentCommand{\RP}{ m }{ \RPaux^{#1} }
\NewDocumentCommand{\CP}{ m }{ \CPaux^{#1} }
\NewDocumentCommand{\gl}{ m }{ \glaux\pa{#1} }
\NewDocumentCommand{\Sp}{ m }{ \Spaux\pa{#1} }
\NewDocumentCommand{\Ort}{ m }{ \Oaux\pa{#1} }
\NewDocumentCommand{\Hip}{ m }{ \Haux^{#1} }
\NewDocumentCommand{\SO}{ m }{ \SOaux\pa{#1} }
\NewDocumentCommand{\SU}{ m }{ \SUaux\pa{#1} }
\NewDocumentCommand{\su}{ m }{ \suaux\pa{#1} }
\NewDocumentCommand{\so}{ m }{ \soaux\pa{#1} }
\NewDocumentCommand{\sym}{ m }{ \symaux\pa{#1} }
\NewDocumentCommand{\Skew}{ m }{ \Skewaux\pa{#1} }

\NewDocumentCommand{\Symp}{ m }{ \Symaux^+\pa{#1} }
\NewDocumentCommand{\TT}{ m }{ \toraux\pa{#1} }
\NewDocumentCommand{\ttalg}{ m }{ \toralg\pa{#1} }
\NewDocumentCommand{\M}{ >{\SplitArgument{1}{,}}m}{%
    \RR^{\prodaux #1}%
}
\NewDocumentCommand{\St}{ >{\SplitArgument{1}{,}}m}{%
    \Staux\pa{\commasaux #1}%
}

\NewDocumentCommand{\Diag}{ >{\SplitArgument{1}{,}}m}{%
    \Diagaux\pa{\commasaux #1}%
}
\NewDocumentCommand{\commasaux}{ m m }{%
    \IfNoValueTF{#2}{ #1 }{ #1, #2 }%
}
\NewDocumentCommand{\prodaux}{ m m }{%
    \IfNoValueTF{#2}{ #1 \times #1 }{ #1 \times #2 }%
}

\makeatletter
\renewcommand\paragraph{\@startsection{paragraph}{4}{\z@}%
                                    {0ex \@plus0.5ex \@minus.2ex}%
                                    {-1em}%
                                    {\normalfont\normalsize\bfseries}}
\makeatother

\usepackage{todonotes}

\title{Trivializations for Gradient-Based \\
Optimization on Manifolds}

\author{%
    Mario Lezcano-Casado\\
  Department of Mathematics\\
  University of Oxford\\
  Oxford, \\
  \texttt{mario.lezcanocasado@maths.ox.ac.uk} \\
}

\begin{document}

\maketitle

\begin{abstract}
We introduce a framework to study the transformation of problems with manifold constraints into unconstrained problems through parametrizations in terms of a Euclidean space.
We call these parametrizations \emph{trivializations}.
We prove conditions under which a trivialization is sound in the context of gradient-based optimization and we show how two large families of trivializations have overall favorable properties, but also suffer from a performance issue.
We then introduce \emph{dynamic trivializations}, which solve this problem, and we show how these form a family of optimization methods that lie between trivializations and Riemannian gradient descent, and combine the benefits of both of them.
We then show how to implement these two families of trivializations in practice for different matrix manifolds. To this end, we prove a formula for the gradient of the exponential of matrices, which can be of practical interest on its own.
Finally, we show how dynamic trivializations improve the performance of existing methods on standard tasks designed to test long-term memory within neural networks.\footnote{An implementation can be found at: \url{https://github.com/Lezcano/expRNN}}

\end{abstract}

\section{Introduction}
Constrained optimization allows to put restrictions on the family of objects being optimized. When the restrictions are simple, for example, having a vector with entries in $[0,1]$ or $[-1,1]$, simple element-wise parametrizations using sigmoid functions or $\tanh$ allow the design of powerful models such as \lstm~\citep{hochreiter1997long} and \gru~\citep{cho2014learning} through the method of \emph{gating}. This kind of vector-regularization is now standard, and most of the advanced neural network architectures use it as a basic building block~\citep{bahdanau2014neural}. Constraints on matrices, on the other hand, are much more challenging.

Most of the interesting sets of matrices turn out to have a manifold structure. Optimization on manifolds is both theoretically and practically challenging due to the inherent complexity of the objects involved. Even then, optimization on matrix manifolds has proven to be rather useful in many different subfields of machine learning and neural networks (\nn). Examples of interesting matrix manifolds in the context of gradient-based optimization are the set of positive definite matrices in Bayesian statistics~\citep{rasmussen2005gaussian}, orthogonal matrices within \rnn s~\citep{arjovsky2016unitary,helfrich18a,lezcano2019cheap}, \nn s with structured linear layers via the \qr{} or the \svd{} decomposition~\citep{berg2018sylvester,zhang2018stabilizing,kingma2018glow}, or invertible matrices in normalizing flows~\citep{berg2018sylvester} and \vae s~\citep{tomczak2016improving}.

In this paper we aim to provide a theoretically sound but also efficiently implementable framework to perform optimization on these and other matrix manifolds in the context of gradient-based optimization.

\paragraph{Outline of the paper and summary of the main contributions}\mbox{}\\
In this paper, we study parametrizations of the form $\deffun{\phi : \RR^n -> \MM;}$.

We consider the transformation of a constrained optimization problem into an unconstrained one.
\[
\text{Initial problem: }
\min_{x \in \MM} f(x) \qquad \quad
\text{Unconstrained problem: }
    \min_{y \in \RR^n} f(\phi(y)).
\]
We call this process \emph{trivialization} and we say that $\phi$ is a trivialization map.
In~\Cref{sec:trivialization}, we show that whenever $\phi$ is regular enough--- a diffeomorphism---these parametrizations act as a change of metric on $\MM$, and thus, applying gradient descent to this new problem is equivalent to performing $\rgd$ on the original problem with this new metric, for which standard convergence results hold.

After this, we look at two large families of parametrizations, the Riemannian exponential, and the Lie exponential. We analyze these from the point of view of the framework presented before, and we point out a problem that they present: they may create saddle points or local minima when near certain region in the manifold.

In~\Cref{sec:dyn_triv}, we introduce \emph{dynamic trivializations}.
They can be described as follows:

\paragraph{Main idea:} Lift the function $f$ to the current tangent space $T_{x_i}\MM$ using a map $\deffun{\phi_{x_i} : T_{x_i}\MM -> \MM;}$ by considering the trivialization $f \circ \phi_{x_i}$ (think $\phi_{x_i} = \exp_{x_i}$, or, for efficiency, any retraction). Optimize $f \circ \phi_{x_i}$ on $T_{x_i}\MM$ for a while using any standard optimization methods like \adam, \rmsprop, or \adagrad, since $T_{x_i}\MM$ is a linear space. When we are at a point $y_k \in T_{x_i}\MM$ on which $\phi_{x_i}$ might create saddle-points or local minima, then we consider the current point in the manifold $x_{i+1} \defi \phi_{x_i}(y_k)$ and we start optimizing the function $f \circ \phi_{x_{i+1}}$, \ie, lift the problem to $T_{x_{i+1}}\MM$.

This family of methods has Riemannian gradient descent and classic trivializations as limit cases, and in particular, they combine the strengths of the two. Furthermore, we show that these methods give a natural generalization of Euclidean optimizers to manifolds.

In~\Cref{sec:impl} we show how to compute the gradients associated to the Lie exponential and some cases of the Riemannian exponential for matrix manifolds. To this end, we compute a formula that allows for the approximation of the gradient of the exponential of matrices to machine-precision. We also show some examples of for how to use this theory to perform optimization on some matrix manifolds. In~\Cref{sec:examples} we compile an extended list of examples that we hope might be helpful to the reader.

Finally, in~\Cref{sec:experiments} we show how dynamic trivializations improve previously developed optimization techniques in the context of optimization with orthogonal constraints.

\section{Related Work}
\paragraph{Optimization on manifolds.} Most of the results on optimization on manifolds have found analogues in the Riemannian setting~\citep{udriste1994convex,absil2009optimization}.
Algorithms like conjugate gradient descent or the Newton method were first devised for specific families of manifolds~\citep{Smith:1993:GOM:165579,edelman1998geometry}, and then they were derived for general Riemannian manifolds~\citep{bonnabel2013stochastic,sato2015new,boumal2016global}.

Optimization methods on manifolds can be classified in two families: Those that follow geodesics, and those that follow retractions---\ie, first order approximations to geodesics. In the first family, convergence rates have been proven for most first order methods, both stochastic and non-stochastic~\citep{zhang2016first}, and even purely first-order accelerated methods~\citep{zhang2018towards}. When it comes to retractions, rates of convergence have been proved in the Lipschitz setting for first and second-order methods~\citep{boumal2016global}.

\paragraph{Trivialization.}
The trick of parametrizing a Lie group with elements in the Lie algebra through the Lie exponential map has been commonly used under the name of \emph{trivialization} in the area of differential equations on manifolds~\citep{magnus1954exponential,iserles1999solution,iserles2000lie}. We borrow the term, as the general idea behind these methods and ours is rather similar.

\paragraph{Optimization through parametrizations.}
Parametrizing a manifold in terms of a Euclidean space is a common technique in optimization and machine learning. For example when doing computations on symmetric positive definite matrices~\citep{arsigny2006log,arsigny2007geometric}, compact Lie groups~\citep{lezcano2019cheap}, the special orthogonal group~\citep{helfrich18a} or the unitary group~\citep{jing2017tunable,maduranga2018complex}. In~\citep{dreisigmeyer2018direct}, it is used through the Riemannian exponential to adapt 0\textsuperscript{th} order methods to naturally reductive homogeneous manifolds.

Our work finds the closest connections in the papers~\citep{lezcano2019cheap,helfrich18a,maduranga2018complex} These papers present the use of the Lie exponential and the Cayley map for optimization on $\SO{n}$. Our framework can be seen as an extension that can be implemented on top of them at a negligible execution cost. We also show that this theoretical improvement translates into a better convergence in practice in~\Cref{sec:experiments}.

\section{Problem Set-Up}
We include a short introduction to the concepts used from differential and Riemannian geometry in~\Cref{sec:diff_geo}.

We are interested in approximating the following problem over a connected manifold $\MM$
\[
    \min_{x \in \MM} f(x).
\]
A differentiable manifold does not carry intrinsically any metric information. As such, if one is interested in talking about concepts like the distance to the optimum, or the steepest descent direction, it is necessary to put additional structure on the problem.
One way to do this is to consider a Riemannian metric $\gm$ on $\MM$, turning $\MM$ into a Riemannian manifold.

\subsection{The classic approach: Riemannian gradient descent}
Given a complete metric on $\MM$, we can define geodesics $\deffun{\gamma_{p, v} : [0,\infty) -> \MM;}$ such that $\gamma_{p,v}(0) = p$, $\gamma'_{p,v}(0) = v$ for $v \in T_p\MM$. Then, the Riemannian exponential map is defined simply as the map that maps rays starting at the origin in the tangent space to geodesics on $\MM$. In symbols, $\exp_p(tv) \defi \gamma_{p,v}(t)$ for $t \geq 0$.

Using the Riemannian exponential, one can define Riemannian gradient descent in an analogous way to the Euclidean case:
\[
    x_{t+1} = \exp_{x_t}(-\eta \grad f(x_t)).
\]
In plain words, the algorithm follows the geodesic defined by the direction of steepest descent $-\grad f(x_t)$ for a time $\eta > 0$. This approach has been extensively studied in the literature and it has been proven to enjoy similar convergence properties to its Euclidean counterpart~\citep{absil2009optimization,bonnabel2013stochastic,boumal2016global,zhang2016riemannian}.

Sometimes it is convenient, due to computational constraints, to use a first order approximation to the exponential rather than the exponential map. This idea is encapsulated in the concept of a retraction.
\begin{definition}[Retraction]\label{def:retraction}
    A differentiable map $\deffun{r : T\MM -> \MM;}$ is called a retraction if for every $p \in \MM$, the map $\deffun{r_p : T_p\MM -> \MM;}$ satisfies $r_p(0) = p$ and $\pa{\dif r_p}_0 = \Id$.
\end{definition}
The update rule of Riemannian gradient descent along a retraction $r$ is then given by
\[
    x_{t+1} = r_{x_t}(-\eta \grad f(x_t)).
\]
In many cases, this update rule is enough to have the same convergence properties as in Riemannian gradient descent along the exponential map~\citep{boumal2016global}.

The main problem of Riemannian gradient descent comes from a practical point of view. On many practical problems, it has been empirically proved that algorithms like \adam~\citep{kingma2014adam}, \adagrad~\citep{duchi2011adaptive} or \rmsprop~\citep{tieleman2012lecture} outperform vanilla \sgd. These algorithms were designed to work on $\RR^n$, and although generalizations for product manifolds are in order \citep[\cf,][]{becigneul2018riemannian}, it is not clear how to generalize them to most manifolds used in practice, and thus take advantage of them in the Riemannian setting.

\section{Trivializations}\label{sec:trivialization}
We now introduce trivializations. Trivializations are functions that allow us to transform a constrained problem on a manifold to an unconstrained one.

\begin{definition}[Trivialization]
    Given a manifold $\MM$, we define a trivialization as a surjective map
    \[
        \deffun{\phi : \RR^n -> \MM;}.
    \]
\end{definition}

\begin{example}
    The most simple examples are found when $\MM$ has a product structure, \ie, for vectors. For example, for a fixed $n >0$, consider component-wise functions like rectified linear units, parametrizing non-negative vectors $\deffun{\fn{relu} : \RR^n -> \pa{\RR^+}^n;}$ or the sigmoid function $\deffun{\sigma : \RR^n -> [0,1]^n;}$.
\end{example}

Having a trivialization in hand, we can transform a constrained optimization problem into an unconstrained one by composing $f$ with $\phi$.
\[
    \min_{y \in \RR^n} f(\phi(y)).
\]

\begin{remark}
    When considering a parametrization $\phi$, the gradient $\grad f(x)$ changes into the gradient $\grad \pa{f \circ \phi}(y)$ for $x = \phi(y)$. For a $1$-dimensional trivialization, by the chain rule, if $\phi'(y) = 0$ for many $y \in \RR$, $\phi$ will not be a good parametrization, because then $\grad\pa{f \circ \phi}(y) = \grad f (\phi(y)) \phi'(y) = 0$, even though $\grad f(x)$ might not be zero.
    As such, not all trivializations are equally good.
\end{remark}

    We formalize this intuition for general trivializations in the following theorem.

\begin{theorem}\label{thm:change_metric}
    Let $\deffun{\phi : \RR^n -> \MM;}$ be a diffeomorphism. Then, solving the problem $\min_{y \in \RR^n} f(\phi(y))$ through gradient descent accounts for solving the problem $\min_{x \in \MM} f(x)$ using Riemannian gradient descent for a certain metric on $\MM$ induced by $\phi$.
\end{theorem}
\vspace{-0.15in}
\begin{proof}
See~\Cref{sec:parametrizations}.
\end{proof}
\vspace{-0.15in}

This result tells us that, if $\phi$ is a diffeomorphism, $\phi$ will not add local minima or saddle points. It will simply act as a change of metric on the manifold. This already explains the good behavior of the $\tanh$ and sigmoid functions present in an \lstm{} or \gru{} in the context of gating.

At first sight, the situation of $\phi$ being a diffeomorphism seems too restrictive for general manifolds. We now introduce two parametrizations that are diffeomorphisms in \emph{almost all} the manifold.\footnote{This is taken with respect to the canonical Borel measure on the manifold induced by the metric.}

\subsection{The Riemannian trivialization}
Consider now the Riemannian exponential map. By the Hopf-Rinow theorem, it is surjective whenever $(\MM, \gm)$ is connected and complete. As such, in these cases, for any point $p \in \MM$, the Riemannian exponential map $\deffun{\exp_{\MM, p} : T_p\MM \pa{\iso \RR^n} -> \MM;}$ is an example of a trivialization.

\paragraph{Geometric intuition about the Riemannian trivialization.}
A direct corollary of Gauss' lemma says that the metric induced by the exponential parametrization $\exp_{\MM, p}$ is a first order approximation to the metric on the manifold around the point $p$~\citep[\cf{}][Lemma $5.5.7$]{petersen2016riemannian}.
In other words, the Riemannian trivialization changes the metric into a new one with the square of the distance to $p$ for points near $p$.

Let us now look at the behavior of the Riemannian trivialization in global terms.
\begin{theorem}[Properties of the Riemannian trivialization]\label{thm:properties_exp}
    Let $(\MM, g)$ be a connected, complete Riemannian manifold. Fix a point $p \in \MM$. Let $U_p \subset T_p\MM$ be the largest radially convex open neighborhood of zero on which $\exp_{\MM, p}$ is a diffeomorphism\footnote{A more formal way to define it would be $\overline{U}_p \defi \set{v \in T_p\MM | \exp_p(tv) \text{ is length minimizing for }t \in [0,1]}$.}
    then, $\exp_{\MM, p}(\overline{U}_p) = \MM$.

    Furthermore, define the \emph{cut locus in $T_p\MM$} as $\tilde{C}_p \defi \overline{U}_p \backslash U_p$. If  $V \in T_p\MM$ is another open neighborhood of the origin that contains a point in $\tilde{C}_p$, then $\exp_{\MM, p}$ is not a diffeomorphism on $V$.
\end{theorem}
\vspace{-0.15in}
\begin{proof}
    See Section $5.7.3$ in~\citep{petersen2016riemannian}.
\end{proof}
\vspace{-0.15in}

\Cref{thm:properties_exp} combined with~\Cref{thm:change_metric} tell us that there exists a radially convex neighborhood of zero on which $\exp_{\MM, p}$ acts as a change of metric, and that $\exp_{\MM,p}$ stops being a diffeomorphism in the boundary---and hence, can add minima or saddle points at these points. As the image of $\overline{U}_p$ is the whole $\MM$, if we write $C_p \defi \exp_{\MM, p}\pa{\tilde{C}_p}$, we have that $\MM$ decomposes in the disjoint union of $\exp_{\MM, p}(U_p)$ and $C_p$. The set $C_p$ is called the \emph{cut locus of $p$}.

The cut locus is a remarkably slippery object of study given that, in general, it is not differentiable. Nonetheless, we can still measure the relative size of this set in a topological sense, by means of the Hausdorff dimension.
\begin{theorem}[\citet{itoh1998dimension}]
    Let $\MM$ be a connected and complete Riemannian manifold of dimension $n$. For a point $p \in \MM$ the Hausdorff dimension of $\tilde{C}_p$ is either $0$ or $n-1$, and the Hausdorff dimension of $C_p$ is an integer less than $n$.
\end{theorem}

Putting this result in the more familiar language of measures, we can argue that, although the cut locus can introduce problems in practice, the problematic set is not too large.\footnote{The analogous result for $C_p$ with respect to the Borel measure induced by the volume form is also true.}
\begin{corollary}
    $\tilde{C}_p$ has Lebesgue measure zero on $T_p\MM$.
\end{corollary}
\vspace{-0.15in}
\begin{proof}
    By the definition of Hausdorff dimension, a set of dimension $n-1$ has $n$-Hausdorff measure $0$. Finally, just note that the $n$-Hausdorff measure is a multiple of the Lebesgue measure.
\end{proof}
\vspace{-0.15in}

\subsection{The Lie trivialization}\label{sec:lie_trivialization}

We now introduce a useful trivialization for Lie groups and other matrix manifolds.
Recall that for a Lie group $G$ we define its \emph{Lie algebra} as the tangent space to the identity element $\g \defi T_e G$. In Lie group theory there is a canonical trivialization given by the \emph{Lie exponential}.
For matrix Lie groups, which are the groups that we are interested in, the Lie exponential is exactly the exponential of matrices. We will denote the exponential of a matrix $A$ as $\exp(A)$ or $e^A$ for short.

For connected and compact Lie groups---\eg, $\SO{n}, \U{n}, \SU{n}, \Sp{n}$---this map is surjective and it coincides with the Riemannian trivialization at the identity for a suitable metric.
If it is not surjective, we can still use it as a trivialization of the image of $\g$ under $\exp$. In~\Cref{sec:lie_exponential} we explain how to use the exponential parametrization in the whole Lie group, even when it is not surjective.
Trivializations of this form for compact Lie groups were already studied in~\citep{lezcano2019cheap}.

The following theorem is a generalization for matrix Lie groups of a classic result.
\begin{theorem}[Properties of the Lie exponential]\label{thm:lie_exponential}
    Let $G$ be a matrix Lie group, the Lie exponential is a diffeomorphism on the set
    $U = \set{A \in \g | \abs{\Im\pa{\lambda_i(A)}} < \pi}$
    with $\lambda_i(A)$ the eigenvalues.
\end{theorem}
\vspace{-0.15in}
\begin{proof}
    See~\Cref{sec:proof}.
\end{proof}
\vspace{-0.15in}

This result is the counterpart of~\Cref{thm:properties_exp} for the Lie trivialization on general matrix Lie groups. The boundary of this set has similar properties as those of the cut locus for the Riemannian trivialization for groups like $\GL{n}$ or $\SO{n}$.\footnote{The constant $\pi$ is tight for matrix manifolds that contain matrices with eigenvalues that are $2\pi i$ apart. For these manifolds, the matrix exponential fails to be a diffeomorphism on some points of the boundary of $U$.}
As such, this trivialization presents the same problem as the Riemannian trivialization: It works as a change of metric for points that are close to the identity matrix, but it creates local minima and saddle points on some points of the manifold, which we might encounter as the optimization method progresses.

\section{Dynamic Trivializations}\label{sec:dyn_triv}
In the last section, we have seen rather general families of trivializations that cover most of the manifolds used in practice. We have seen how these trivializations act as a change of metric around the initial point---$p$ in the case of the Riemannian trivialization and the identity matrix in the case of the Lie trivialization---but we have also shown that the optimization process can be affected as it deviates from the initial point.

Note that, in the case of the exponential trivialization, we have a map from any tangent space of $\MM$ onto $\MM$, but we are just using one of them as a trivialization. We can leverage the structure of $T\MM$ in order to solve the problem that the trivializations introduced above. Instead of always using $\exp_{\MM, p}$, we can use it just for $K$ optimization steps and then change $p$
to the point on which we find ourselves on the manifold after those $K$ steps. This idea is formalized in the following algorithm.

\begin{samepage}
\begin{algorithmm}[Dynamic trivialization through retractions]\label{alg:dyn_triv}
    Given a retraction $\phi$, an integer $K > 0$ or $K = \infty$, and a starting point $p_0$, the dynamic trivialization induced by $\phi$ is defined as the sequence of problems indexed by $i = 0, 1, \dots$
    \[
        \min_{y \in T_{p_i}\MM} f(\phi_{p_i}(y))
    \]
    where $p_{i+1} \defi \phi_{p_i}(y_{i, K}) \in \MM$, and $y_{i, k}\in T_{p_i}\MM$ for $k=1, \dots, K$, is a sequence of approximations given by a Euclidean optimization algorithm---\eg, \sgd, \adam, \adagrad, \rmsprop, …---applied to the $i$-th problem with starting point $y_{i, 0} = 0$. We say that $p_i$ is the \emph{basis} at step $i$.
\end{algorithmm}
\end{samepage}

\begin{remark}
    Note that in this case we have dropped the condition of $\deffun{\phi_p : T_p\MM -> \MM;}$ being surjective. This is because, as long as $\MM$ is connected, we can still reach any point in $\MM$ in the optimization process by changing the basis of the dynamic trivialization whenever $K < \infty$.
\end{remark}

\ifarxiv
\begin{figure*}[!tbp]
\centering
  \begin{minipage}[b]{0.49\textwidth}
      \centering
      \includegraphics[width=.696\columnwidth]{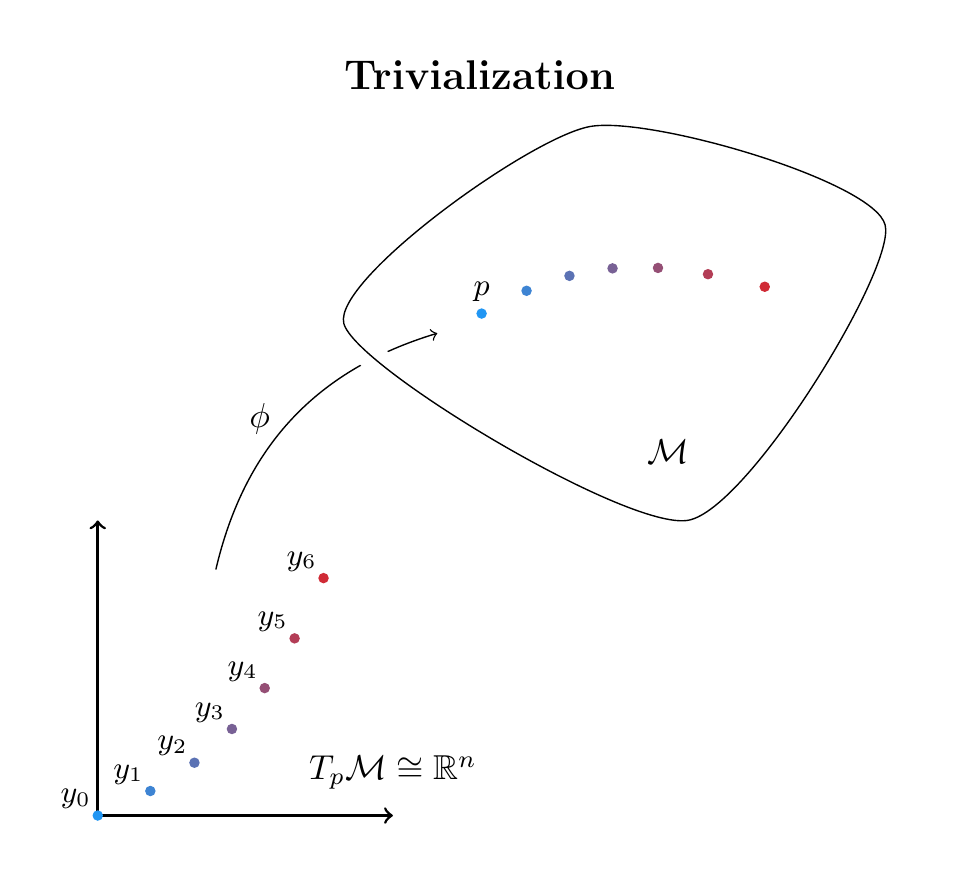}
  \end{minipage}
  \begin{minipage}[b]{0.49\textwidth}
      \centering
      \includegraphics[width=\columnwidth]{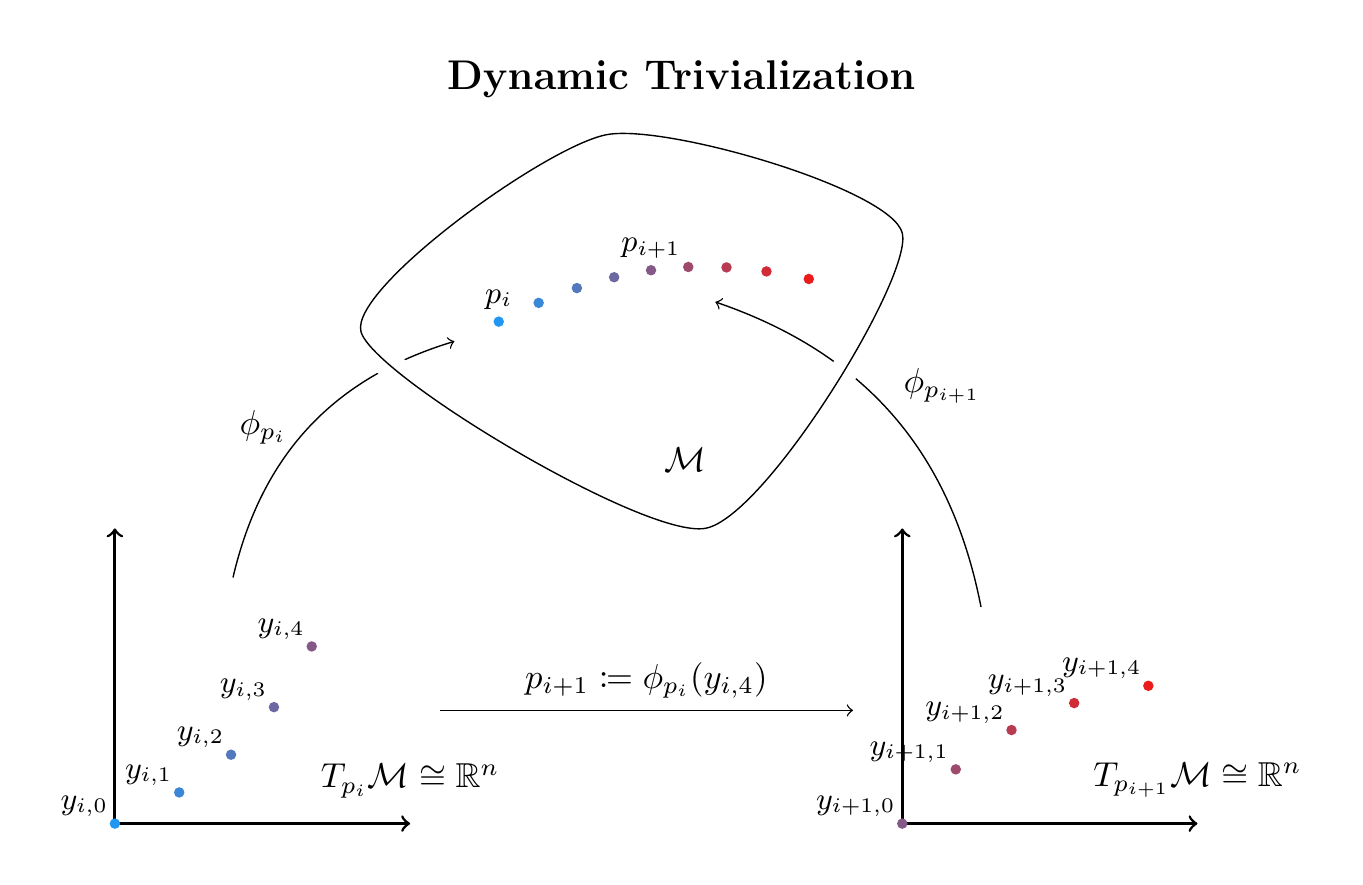}
  \end{minipage}
  \caption{Example of the trivialization and dynamic trivialization procedure. The dynamic trivialization in this example has $K=4$.}
  \label{fig:optimization}
\end{figure*}
\fi

This procedure has two interesting limit cases.

\paragraph{Generalization of trivializations.}
For $K = \infty$, \ie, no change of basis, it reduces to the trivialization algorithms described in~\Cref{sec:trivialization} with the trivialization $\phi_{p_0}$, provided that $\phi_{p_0}$ is surjective.

\paragraph{Generalization of Riemannian gradient descent.}
In the case $K=1$, we are changing the basis of the trivialization on every step. When the optimization process used to generate the iterates $y_{i, k}$ is regular \sgd, this method recovers exactly stochastic Riemannian gradient descent using $\phi$ as a retraction. For this, just note that by the chain rule and the definition of a retraction
\[
    \dif \pa{f \circ \phi_{p_i}}_0
    = \pa{\dif f}_{\phi_{p_i}(0)} \circ \pa{\dif \phi_{p_i}}_0
    = \pa{\dif f}_{\phi_{p_i}(0)}
    = \pa{\dif f}_{p_i}.
\]
From this it follows that
\[
    \grad \pa{f \circ \phi_{p_i}}(0) = \grad f \pa{p_i}
\]
so the update rule simplifies for a learning rate $\eta > 0$ can be rewritten as
\[
    y_{i, 1} = -\eta\grad f(p_i) \qquad p_{i+1} = \phi_{p_i}(-\eta\grad f(p_i))
\]
and $p_{i+1}$ are exactly the iterates given by doing Riemannian \sgd{} using the retraction $\phi$.

In particular, we have proved that for $\phi = \exp_{\MM}$, we recover stochastic Riemannian gradient descent. As such, we can see dynamic trivializations as an interpolation between the trivialization method using $\exp_{\MM}$ and stochastic Riemannian gradient descent.

More interesting is perhaps the case when we use a different optimizer to generate the iterates $y_{i, k}$. In this case, dynamic trivializations yield a natural generalization to manifolds of the algorithm used to generate the iterates, \ie, \adam, \adagrad{}, \rmsprop{}, \etc.

\section{Gradient Computations and Examples}\label{sec:impl}
The last missing piece needed to implement dynamic trivializations is the explicit computation of their gradients. We will do so for the two families presented above.

\subsection{The matrix exponential}
We first look at the matrix exponential. This function not only defines the Lie trivialization, but it is also essential to compute the Riemannian exponential in many matrix manifolds (\cf, \Cref{sec:examples}). In order to implement the dynamic trivialization algorithm within the context of first-order methods we need an approximation of the trivialization map and its gradient.

The current fastest machine-precision approximation to the matrix exponential was formulated in~\citep{al2009new}. On the other hand, it is not clear how to compute the gradient of this parametrization. The following proposition settles this problem.
\begin{proposition}[Gradient of the exponential parametrization]\label{prop:grad_exp}
    Let $\deffun{f : \M{n} -> \RR;}$ be a function defined on matrices, and let $\exp$ be the matrix exponential, we have
    \[
        \grad \pa{f \circ \exp}(A) = \pa{\dif \exp}_{\trans{A}}\pa{\grad f(e^A)}.
    \]
\end{proposition}
\vspace{-0.15in}
\begin{proof}
    See~\Cref{sec:matrix_grad}.
\end{proof}
\vspace{-0.15in}

This proposition together with the approximation algorithm for $\dif \exp$ presented in~\citep{al2009computing} allows us to approximate to machine-precision this gradient.

This formula readily allows for the implementation of the Riemannian dynamic trivialization on many matrix manifolds. We give examples of some of these in~\Cref{sec:examples}.

\subsection{Lie exponential for matrix Lie groups}\label{sec:lie_exponential}
The Lie exponential on a Lie group $G$ is just defined on the Lie algebra $\g = T_eG$. On matrix Lie groups, we can identify any tangent space of $G$ with $\g$. Explicitly, if $\tilde{A} \in T_B G$, then $B^{-1}\tilde{A} \in \g$. Furthermore, if we choose a left-invariant metric on the Lie group, we can then use left multiplication to map the result $\exp(B^{-1}\tilde{A})$ to a neighborhood of $B$. In symbols, we can define
\[
    \deffun{\exp_B : T_B G -> G; \tilde{A} -> B\exp\pa{B^{-1} \tilde{A}}}
\]
We give the gradient of this parametrization in~\Cref{corol:gl}. This function constitutes a dynamic trivialization on any connected matrix Lie group, like, for example, $\SO{n}$, $\U{n}$, $\SL{n}$, or $\GLp{n}$.

\subsection{Other retractions}\label{sec:retractions}
Sometimes one cannot afford to approximate the exponential exactly, as it can be very costly. In this case, the standard alternative are retractions~\cite{boumal2016global}.

\paragraph{Cayley map.}
This is one of the most well known retractions to optimize over $\SO{n}$ (\cf, \cite{absil2009optimization,helfrich18a})
\[
    \deffun{\cay : \Skew{n} -> \SO{n}; A -> (\I + A)(\I - A)^{-1}}
\]
This can be made into a dynamic retraction using the same trick as we did with the exponential, considering $\cay_B(A) = B\cay(B^{-1}\tilde{A})$, for $B \in \SO{n}$, $\tilde{A} \in T_B\SO{n}$.

\paragraph{Projectors.}
Another common retraction used in matrix manifolds $\MM \subset \RR^{n \times n}$ is the one given by $\pi_{\MM}(x+v)$ for $x \in \MM$, $v \in T_x\MM$ and $\pi_{\MM}$ the projection from $\RR^{n \times n}$ onto $\MM$. For example, for $\MM = \SO{n}$, we have that for a matrix $B \in \RR^{n \times n}$ with SVD decomposition $B = U\Sigma\trans{V}$, its projection onto $\SO{n}$ is given by $\pi_{\SO{n}}(B) = U\trans{V}$.\footnote{Formally, $\pi_{\SO{n}}$ is well-defined for matrices such that $\det B>0$, that is, $\deffun{\pi_{\SO{n}} : \GLp{n} -> \SO{n};}$. Note that this function is not a diffeomorphism but a submersion. \Cref{thm:change_metric} can be extended to this case.} with gradient computed in~\citep[\cf{}][Eq. $2.18$]{kenney1991polar}.

We workout more useful examples for common manifolds in~\Cref{sec:examples}.

\section{Experiments}\label{sec:experiments}
In this section, we assess the effectiveness of dynamic trivializations (\dtriv) in the context of orthogonal optimization. We test the framework with the basis changed every $K = 1, 100, \infty$ steps.

We compare it against the most performant previous approaches presented for this task in the context of orthogonal optimization and a vanilla \lstm. These approaches are orthogonal exponential trivialization~\citep[\exprnn{}][]{lezcano2019cheap}, orthogonal and unitary Cayley trivializations \citep[\scornn{} / \scurnn][]{helfrich18a,maduranga2018complex}, and Riemannian gradient descent \citep[\rgd][]{wisdom2016full}.

The architecture on which we are testing the dynamic trivialization is the same as in the papers above: A vanilla \rnn{} with an orthogonal layer parametrized using the Lie trivialization (\cf, \Cref{sec:lie_exponential})
\[
    h_{t+1} = \sigma\pa{\exp_B(A) h_t + Tx_{t+1}}.
\]
The update procedure for $B$ was described in~\Cref{alg:dyn_triv} ($K = 1, 100, \infty$).

\begin{remark}
    Note that $\rgd$ is equivalent to \dtriv{}$1$ together with the optimizer \sgd. Furthermore, \exprnn{} is equivalent \dtriv{}$\infty$ only that \exprnn{} has the basis on the identity matrix and \dtriv{}$\infty$ has the basis on the matrix to which it is initialized.
\end{remark}

We test this architecture on two different tasks that have become the standard to test the performance of \rnn s in the context of long-term recall and long-term memory, namely the pixel-by-pixel \mnist{} and the \timit{} dataset~\citep{arjovsky2016unitary,henaff2016recurrent,mhammedi2017efficient,helfrich18a,maduranga2018complex,lezcano2019cheap}. We do not present results for the copying problem, as task is too simple to draw any meaningful conclusions, as explained in~\citet{henaff2016recurrent}. \footnote{For reference, dynamic trivializations are also able to converge to the correct answer stably, as \exprnn.}

We detail all the hyperparameters and set-up in~\Cref{sec:hyperparam}. The code and instructions to replicate these experiments can be found in
\begin{table}[t]
    \centering
\begin{minipage}[b]{0.40\columnwidth}
\caption{Best test accuracy at \mnist{} and \pmnist{}.}
\label{tab:mnist}
\begin{small}
\begin{sc}
\begin{tabular}{l c c c}
    \toprule
    Model & n & \mnist & \pmnist \\
    \midrule
    \midrule
    \dtriv{}$1$ & $170$ & $\mathbf{98.3}$ & $\mathbf{95.2}$ \\
    \dtriv{}$100$ & $170$ & $98.2$ & $95.1$ \\
    \dtriv{}$\infty$ & $170$ & $98.1$ & $95.0$ \\
    \exprnn & $170$ & $98.0$ & $94.9$ \\
    \scornn & $170$ & $97.2$ & $94.8$ \\
    \scurnn & $116$ & $97.6$ & $94.9$ \\
    \lstm & $128$ & $81.9$ & $79.5$ \\
    \rgd & $116$ & $94.7$ & $92.5$ \\
    \midrule
    \dtriv{}$1$ & $360$ &  $98.4$ & $96.3$ \\
    \dtriv{}$100$ & $360$ &  $98.8$& $96.4$ \\
    \dtriv{}$\infty$ & $360$ &  $\mathbf{98.9}$ & $\mathbf{96.5}$ \\
    \exprnn & $360$ & $98.4$ & $96.2$ \\
    \scornn & $360$ & $98.1$ & $95.9$ \\
    \scurnn & $250$ & $98.3$ & $96.2$ \\
    \lstm & $256$ & $88.8$ & $88.8$ \\
    \rgd & $256$ & $96.1$ & $93.9$ \\
    \midrule
    \dtriv{}$1$ & $512$ & $98.7$ & $96.7$ \\
    \dtriv{}$100$ & $512$ & $\mathbf{99.1}$ & $96.7$ \\
    \dtriv{}$\infty$ & $512$ & $99.0$ & $\mathbf{96.8}$ \\
    \exprnn & $512$ & $98.7$ & $96.6$ \\
    \scornn & $512$ & $98.2$ & $96.5$ \\
    \lstm & $512$ & $91.9$ & $91.8$ \\
    \rgd & $512$ & $97.3$ & $94.7$ \\
    \bottomrule
\end{tabular}
\end{sc}
\end{small}
\end{minipage}
\hspace{0.05\columnwidth}
\begin{minipage}[b]{0.45\columnwidth}
\caption{Test \mse{} at the end of the epoch with the lowest validation \mse{} for the \timit{} task.}
\label{tab:timit}
\begin{small}
\begin{sc}
\begin{tabular}{l c c c c}
    \toprule
    Model & n & Val. \mse{} & Test \mse{}\\
    \midrule
    \midrule
    \dtriv{}$1$ & $224$ & $6.55$ & $6.54$ \\
    \dtriv{}$100$ & $224$ & $4.80$ & $4.77$ \\
    \dtriv{}$\infty$ & $224$ & $\mathbf{4.75}$ & $\mathbf{4.71}$ \\
    \exprnn & $224$ & $5.34$ & $5.30$ \\
    \scornn & $224$ & $9.26$ & $8.50$ \\
    \scurnn & $128$ & $9.42$ & $7.23$ \\
    \lstm & $84$ &  $15.42$ & $14.30$ \\
    \rgd & $128$ &  $15.07$ & $14.58$ \\
    \midrule
    \dtriv{}$1$ & $322$ & $4.56$ & $4.55$ \\
    \dtriv{}$100$ & $322$ & $3.80$ & $3.76$ \\
    \dtriv{}$\infty$ & $322$ & $\mathbf{3.39}$ & $\mathbf{3.76}$ \\
    \exprnn & $322$ & $4.42$ & $4.38$ \\
    \scornn & $322$ & $8.48$ & $7.82$ \\
    \lstm & $120$ & $13.93$ & $12.95$ \\
    \rgd & $192$ & $15.10$ & $14.50$ \\
    \midrule
    \dtriv{}$1$ & $425$ & $4.21$ & $4.17$ \\
    \dtriv{}$100$ & $425$ & $2.02$ & $1.99$ \\
    \dtriv{}$\infty$ & $425$ & $\mathbf{2.00}$ & $\mathbf{1.97}$ \\
    \exprnn & $425$ & $5.52$ & $5.48$ \\
    \scornn & $425$ & $7.97$ & $7.36$ \\
    \scurnn & $258$ & $4.40$ & $3.39$ \\
    \lstm & $158$ & $13.66$ & $12.62$ \\
    \rgd & $256$ & $14.96$ & $14.69$ \\
    \bottomrule
\end{tabular}
\end{sc}
\end{small}
\end{minipage}
\end{table}

\ifarxiv
\begin{center}
\fi
\url{https://github.com/Lezcano/expRNN}
\ifarxiv
\end{center}
\fi

\subsection{Pixel-by-pixel \mnist}
This task consists of classifying the hand-written images of numbers in the \mnist{} dataset~\citep{lecun-mnisthandwrittendigit-2010} by processing them as a sequence pixel-by-pixel. Each image has $28 \times 28$ pixels, so the sequences are of length $784$. The \emph{unpermuted task} (\mnist) processes the row-by-row flattened image, the \emph{permuted task} (\pmnist) samples a permutation of size $784$ at the beginning and then uses it to permute all the images after flattening them. This task was introduced in~\citet{le2015simple}.

\Cref{tab:mnist} is structured so that architectures with the same number of parameters are compared together.
As we can see, the addition of any dynamic trivialization to the Lie parametrization improves the results on this experiment by $0.4\%$ out of the $1.3\%$ possible in the largest size. Moreover, it always improves the previous results, suggesting that it is always a better option to use dynamic trivializations rather than just plain trivializations. In general, we saw that $\dtriv{}100$ and $\dtriv{}\infty$ gave the highest stability and the best results across the experiments.

\subsection{\timit{} speech dataset}
The \timit{} dataset~\citep{garofolo1992timit} is a set of variable-length real-world speech recordings. These recordings are first downsampled to $8$kHz and then transformed into log-magnitudes via a short-time Fourier transform, giving sequences of $129$ complex numbers per step, and a variable length between $61$ and $490$. The task consists of predicting the next log-magnitude given the previous ones. This experiment was introduced in~\citet{wisdom2016full}.

In this experiment we see a similar behavior of the dynamic trivializations as the one already seen in the \mnist{} and \pmnist{} experiments. It also happens in this experiment that \dtriv{}$100$ and \dtriv{}$\infty$ always improve the performance of their static counterparts with base at the identity and of \rgd.

In the experiments in \scurnn{} they explicitly mention that they are computing the \mse{} without discarding the zeros used to pad the variable-length sequences~\citep{maduranga2018complex}. As such, when computing the \mse, they are dividing by an incorrect number---the longest element in the batch times the elements in the batch---rather than by the correct one---the sum of the lengths of all the elements in the batch. We computed the correct validation and test loss in~\Cref{tab:timit}.

\section{Conclusion and Future Work}
In this paper we have presented a novel way to perform optimization on manifolds that combines the strengths of the two most popular optimization techniques used in machine learning and neural networks---parametrizations and Riemannian gradient descent. We have shown that, by moving the initial point of the parametrization, as the metric is distorted less from the Euclidean one, we can achieve an improvement on the convergence of the neural network.

We leave open an interesting line of research based on applying dynamic trivializations to allow optimization on other interesting manifolds. As a first step in this direction, we detail examples of some computations for the most common manifolds used in optimization in~\Cref{sec:examples}.

\section*{Acknowledgements}
We would like to thank the help of Jaime Mendizabal and Momchil Konstantinov for the very useful feedback and suggestions and Prof.\ Andras Juhasz for the computing power.

The work of MLC was supported by the Oxford-James Martin Graduate Scholarship and the ``la Caixa'' Banking Foundation (LCF/BQ/EU17/11590067).

\bibliography{refs}
\clearpage
\appendix
\section{Differential and Riemannian Geometry}\label{sec:diff_geo}
In this section we give a short introduction to the concepts used in the paper and in the appendix of the theories of differential and Riemannian geometry and Lie groups. The standard modern introduction to differential geometry is~\citet{lee2013introduction}. This book also gives an introduction to Lie groups. Introductory texts in Riemannian geometry are~\citet{do1992riemannian,lee2018introduction}. Introductory references for Lie groups are~\citet{rossmann2006lie,hall2015lie}. Although not covered in this summary, two more advanced texts that cover the classical theory of the cut locus through Jacobi fields are~\citet{gallot2012riemannian,petersen2016riemannian}.

\subsection{Differential Geometry}
Let $\MM$ be an $n$-dimensional differentiable real manifold.
$\MM$ has an associated global object called the \emph{tangent bundle} $T\MM \defi \sqcup_{p \in \MM} \set{p} \times T_p\MM$, that is, the disjoint union of all the tangent spaces at every point of $\MM$. The tangent bundle comes with a structure of a $2n$-dimensional differentiable manifold. A point in $T\MM$ consists then of a pair $(p, v)$ with $p \in \MM$ and $v \in T_p\MM$.
On each point, we also have the \emph{cotangent space} $T_p^\ast\MM$ of linear applications from vectors onto the real numbers. The disjoint union of all the cotangent spaces is another manifold $T^\ast\MM$ called the \emph{cotangent bundle}.
When considering these bundles, tangent spaces $T_p\MM$ and cotangent spaces $T_p^\ast\MM$ are sometimes called \emph{fibres}.

An \emph{affine connection} $\conn$ is a bilinear form that, given two vector fields $X, Y$, assigns a new one $\conn_X Y$, and it is tensorial on the first component and Leibnitz on the second. An affine connection defines a notion of \emph{parallel vector fields}. We say that a vector field $Z$ is parallel along a curve $\deffun{\gamma : [0,1] -> \MM;}$ if $\conn_{\gamma'} Z = 0$ where $\gamma' \defi \dif \gamma\pa{\frac{\dif}{\dif t}}$. For any curve, given an initial vector $Z_0$, there exists a unique parallel vector field $Z$ along it such that $Z(0) = Z_0$. We say that the vector $Z(t)$ is the parallel transport of $Z(0)$ for $t \in [0, \epsilon)$.

\subsection{Riemannian Geometry}
A \emph{Riemannian manifold} is a differentiable manifold together with a smooth metric $\deffun{\gm_p : T_p\MM \times T_p\MM -> \RR;}$ which is symmetric and positive definite. A metric induces a distinguished connection called the \emph{Levi-Civita connection}. This is the unique connection that is torsion-free, $\conn_X Y - \conn_Y X = [X, Y] \defi XY -YX$, and it is compatible with the metric, $\mathrm{D}_Z(\gm(X, Y)) = \gm(\conn_Z X, Y) + \gm(X, \conn_Z Y)$, where $\mathrm{D}_Z$ denotes the directional derivative in the direction of $Z$. Whenever we talk about a connection on a Riemannian manifold we will always be referring to the Levi-Civita connection.

A Riemannian manifold has a notion of length of a differentiable curve $\deffun{c : [0,1] -> \MM;}$, $L(c) = \int_0^1 \norm{\gamma'(t)}\,\dif t$. When the manifold is connected, this allows to put the structure of a metric space on the manifold, defining the distance between two points as the length of the shortest piece-wise differentiable curve joining these two points.

Given a connection, we define a \emph{geodesic} $\deffun{\gamma : [0, \epsilon) -> \MM;}$ as a self-parallel curve, $\conn_{\gamma'} \gamma' = 0$. Geodesics are defined for any starting conditions $(p, v) \in T\MM$, $\gamma(0) = p$, $\gamma'(0) = v$ on an interval $[0, \epsilon)$. If a Riemannian manifold is connected and complete, the \emph{Hopf-Rinow theorem} asserts that geodesics not only exist locally, but globally, that is, they can be extended indefinitely taking $\epsilon = \infty$ giving $\deffun{\gamma : [0, \infty) -> \MM;}$. Furthermore, Hopf-Rinow adds that, under the same conditions, there exists a geodesic connecting any two given points. When the connection comes from a metric, geodesics are the locally length-minimizing curves on $\MM$.

Given a connection, we define the exponential map as $\exp_p(v) \defi \gamma_{p,v}(1)$ where $\gamma_{p,v}$ is the geodesic with initial conditions $(p, v)$. On a connected and complete Riemannian manifold, Hopf-Rinow says that the exponential map is defined in the whole tangent bundle.

A metric induces an isomorphism between the tangent and cotangent bundle $\deffun{\alpha : T\MM -> T^\ast\MM;}$ defined as $\alpha(X) \defi \gm(X, -)$. $\alpha$ is sometimes called the \emph{musical isomorphism}. The gradient of a function is defined as the vector field associated to the differential form $\dif f$ through this isomorphism $\grad f \defi \alpha^{-1}(\dif f)$. In other words, it is the vector field such that $\dif f = \gm(\grad f, -)$. As such, the gradient depends on the choice of metric. A metric also allows to define the adjoint of a differential $\deffun{\dif \phi : T_p\MM -> T_{\phi(p)}\MM;}$ at a point $p \in \MM$ as the application $\deffun{\dif \phi^\ast : T_{\phi(p)}\MM -> T_p\MM;}$ such that for every $X\in T_p \MM, Y \in T_{\phi(p)}\MM$ we have that $\gm(\dif \phi(X), Y)_{\phi(p)} = \gm(X, \dif \phi^\ast(Y))_p$.

\subsection{Lie groups}
A Lie group $G$ is a differentiable manifold equipped with a differentiable group structure. Lie groups have a distinguished tangent space called the \emph{Lie algebra}, which is the tangent space at the identity $\g \defi T_eG$. Any closed subgroup of a Lie group is itself a Lie group. A (real) \emph{matrix manifold} is a closed subgroup of the \emph{general linear group} $\GL{n} = \set{B \in \M{n} | \det A \neq 0}$. The Lie algebra of the general linear group is $\gl{n} = \M{n}$. In general, the general linear group of a vector space $V$  is the Lie group formed by the invertible automorphisms of $V$, $\GL{V}$.

On a Lie group, one has for every $g, x \in G$ the diffeomorphisms given by left translations $L_g(x) \defi gx$, right translations $R_g(x) \defi xg$, and conjugation $c_g(x) = gxg^{-1}$. Using left translations, one can identify any tangent space with the Lie algebra via the vector space isomorphism $\deffun{\pa{\dif L_{g^{-1}}}_g : T_gG -> \g;}$. The differential of the conjugation at the identity is called the \emph{adjoint representation of $G$}, $\deffun{\Ad : G -> \GL{\g};}$. The differential of $\Ad$ at the identity is the \emph{adjoint representation of $\g$}, $\deffun{\ad : \g -> \End{\g};}$. For matrix Lie groups, $\Ad_g(X) = gXg^{-1}$ and $\ad_X(Y) = [X, Y]$.

Given a vector $X \in \g$, we can consider the \emph{one parameter subgroup} with starting vector $X$, which is the unique group homomorphism $\deffun{\gamma_X : \RR -> G;}$ such that $\gamma_X'(0) = X$. The Lie exponential is then defined for every $X \in \g$ as $\exp(X) \defi \gamma_X(1)$. For matrix Lie groups, the Lie exponential is given by the exponential of matrices.

A Riemannian metric on a Lie group is said to be left (resp.\ right) invariant if it turns left (resp.\  right) translations into isometries. A metric is said to be bi-invariant if it is both left and right invariant. Every Lie group admits a left-invariant metric, given by choosing any inner product in $\g$ and pushing it forward using $L_{g^{-1}}$. Only compact Lie groups, commutative Lie groups, and products of them admit bi-invariant metrics. When a Lie group is equipped with a bi-invariant metric, the Lie exponential coincides with the Riemannian exponential at the identity.

\section{Parametrizations on Manifolds}\label{sec:parametrizations}
In this section we look at the problem of how does optimizing $f \circ \phi$ affect the optimization problem, depending on the properties of $\phi$. As a disclaimer we would like to mention that, although this section and next section are original, most of them would be considered routine in the field of differential geometry.

Consider the optimization problem
\begin{equation}\label{eq:3_min_problem}
    \min_{x \in \MM} f(x)
\end{equation}
where $\MM$ is a Riemannian manifold.
In this section we will look at parametrizations, which can be regarded as a generalization of certain trivializations, when the domain is not necessary $\RR^n$ but a Riemannian manifold.

Suppose that we have access to a diffeomorphism between Riemannian manifolds
\[
\deffun{\phi : \NN -> \MM;}.
\]
and denote the metric on $\NN$ as $\gm_2$. We say that $\phi$ is a \emph{parametrization of $\MM$ in terms of $\NN$}.

We can then consider the problem
\[
    \min_{y \in \NN} f(\phi(y)).
\]
In order to apply a first-order method to this new problem we first have to compute the gradient of this new function $f \circ \phi$. In order to do so, let us first define some notation.

Denote by $\dif\phi$ and $\dif\phi'$ the differential and its dual
\begin{align*}
    \deffun{\dif\phi &: T\NN -> T\MM;} \\
    \deffun{\dif\phi' &: T^\ast\MM -> T^\ast\NN;}
\end{align*}
and denote by $\alpha$ and $\beta$ the canonical isomorphisms between the tangent and the cotangent bundle induced by the metrics
\begin{align*}
    \alpha &\colon T\MM \overset{\iso}{\to} T^\ast\MM\\
    \beta &\colon T\NN \overset{\iso}{\to} T^\ast\NN.
\end{align*}
Finally, denote by $\dif\phi^\ast$ the fibre-wise adjoint with respect to the two metrics of $\dif\phi$
\[
    \deffun{\dif\phi^\ast : T\MM -> T\NN;}.
\]

\begin{proposition}\label{prop:3_dual_adjoint}
    Using the notation above, the following relation holds
    \[
    \beta \circ \dif\phi^\ast = \dif\phi' \circ \alpha.
    \]
\end{proposition}
\begin{proof}
    For $Y \in T\NN, X \in T\MM$, we have that
    \[
        \pa{\dif\phi' \circ \alpha}(X)(Y) = \alpha(\dif\phi(Y))(X) = \beta\pa{\dif\phi^\ast(X)}\pa{Y} = \pa{\beta \circ \dif\phi^\ast}(X)(Y).\qedhere
    \]
\end{proof}

Using this proposition, we can compute the gradient with respect to the new parametrization.

\begin{corollary}\label{corol:3_grad_abstract}
    Let $\deffun{\phi : \NN -> \MM;}$ be a smooth map between Riemannian manifolds and $f$ be a function on $\MM$. We have that
    \[
        \grad\pa{f \circ \phi} = \dif\phi^\ast(\grad f).
    \]
\end{corollary}
\begin{proof}
    This is direct using the previous proposition since
    \[
        \grad\pa{f \circ \phi} := \beta^{-1}\pa{\dif \pa{f \circ \phi}} = \pa{\beta^{-1} \circ \dif\phi'}\pa{\dif f} = \dif\phi^\ast\pa{\grad f}.\qedhere
    \]
\end{proof}

This motivates the definition of the metric associated to a parametrization $\phi$.
\begin{definition}[Metric associated to a parametrization]\label{def:3_metric_parametrization}
    A parametrization between Riemannian manifolds $\deffun{\phi : \NN -> \MM;}$ induces a metric on $\MM$ as per
\[
    \pa{\phi_\ast \gm_2}(X_1, X_2)_p := \gm_2\pa{\dif\phi^\ast(X_1), \dif\phi^\ast(X_2)}_{\phi^{-1}(p)}\qquad \forall p \in \MM.
\]
\end{definition}

This is a metric since $\dif\phi^\ast(X) = 0$ if and only if $X=0$ by the inverse function theorem, given that $\phi$ is a diffeomorphism.

Another way of looking at this construction is through the lens of submersions.
\begin{definition}[Riemannian Submersion]
    A Riemannian submersion is a surjective map $\deffun{\phi : \NN -> \MM;}$ such that its differential is surjective at every point and
    \[
        \deffun{\dif \phi : \pa{\ker\pa{\dif \phi}}^\perp  -> T\MM;}
    \]
    is an isometry.
\end{definition}
This is equivalent to saying that the adjoint $\dif \phi^\ast$ should be an isometry. This is exactly the construction that we are using, we take the metric that converts $\phi$ into a Riemannian submersion.

We now look at this new metric. We will prove that doing gradient descent using a retraction along $\phi_\ast \gm_2$, is not a retraction with respect to $\gm_2$, and hence, it constitutes an optimization method fundamentally different to the original Riemannian gradient descent.

Using this metric, gradient descent on $\MM$ with a step-size $\eta > 0$ is given by the map
\[
    y_{t+1} = \pa{\phi \circ \exp_{\NN, \gm_2} \circ \dif\phi^\ast}\pa{-\eta\grad f(y_t)}
\]
where $\deffun{\exp_{\NN, \gm_2} : T\NN -> \NN;}$ is the Riemannian exponential map on $(\NN, \gm_2)$. Note that since $\grad f = \alpha^{-1} \circ \dif f$, this step does not depend on the initial metric on $\MM$, as we already observed in the proof of~\Cref{corol:3_grad_abstract}.

More generally, recall the definition of a retraction.
\begin{definition}[Retraction]
    A differentiable map $\deffun{r : T\NN -> \NN;}$ is called a retraction if for every $p \in \NN$
    \[
        r_p(0) = p \qquad\text{and}\qquad \pa{\dif r_p}_0 = \Id.
    \]
    In other words, $r$ is an order one approximation to the Riemannian exponential.
\end{definition}

As proved in~\citet{boumal2016global}, under Lipschitzness conditions, it is enough to follow retractions rather than the exponential map in order to achieve convergence to a local minimum with Riemannian gradient descent. As such, a natural question to ask is whether the function that defines the update step defines a retraction.

\begin{proposition}\label{prop:3_retraction_step}
    Let $\pa{\MM, \gm_1}, \pa{\NN, \gm_2}$ be Riemannian manifolds. Let $\phi$ be a parametrization between them and let $\deffun{r : T\NN -> \NN;}$ be a retraction. The map
    \[
        \deffun{\phi_\ast r \defi \phi \circ r \circ \dif\phi^\ast : T\MM -> \MM;}
    \]
    is a retraction if and only if $\phi$ is a local isometry.
\end{proposition}
\begin{proof}
    It is clear that $\pa{\phi_\ast r}_p(0) = p$. For the second condition, differentiating, we have that the map is a retraction if and only if
    \[
        \dif\phi \circ \dif\phi^\ast = \Id_{T_p\MM}.
    \]
    or equivalently $\dif\phi^{-1} = \dif\phi^\ast$. Now,
    \[
        \pa{\dif\phi \circ \dif\phi^\ast \circ \dif\phi \circ \dif\phi^\ast}_p = \Id_{T_p\MM}
    \]
    so
    \[
        \pa{\dif\phi^\ast \circ \dif\phi}_{\phi^{-1}(p)} = \pa{\dif\phi^{-1} \circ \dif\phi^\ast}_{\phi^{-1}(p)} = \Id_{T_{\phi^{-1}(p)}\NN}.
    \]

    Finally, since $\dif\phi^\ast$ is the adjoint operator with respect to the metrics $\gm_2$ and $\gm_1$, evaluating this last expression on two points using the metric
    \[
        \gm_1\pa{\dif\phi(u), \dif\phi(v)} = \gm_2(u, v) \qquad \forall u, v \in T_{\phi^{-1}(p)} \NN,
    \]
    which is equivalent to $\phi$ being a local isometry.
\end{proof}

This is not a surprising result, since a retraction is a map that preserves the gradient. The way we have defined $\phi_\ast r$ is such that it preserves the gradient with respect to $\gm_2$. If it also preserved the gradient with respect to $\gm_1$, that would mean that the gradients with respect to the two metrics are the same, modulo a transformation through $\dif\phi^\ast$, in other words, $\dif\phi$ should be a local isometry.

\section{Proof of~\Cref{thm:lie_exponential}}\label{sec:proof}
In this section we generalize to general matrix Lie groups the classic proof presented in Theorem D.$2.$ in~\citet{lezcano2019cheap}.

In order to generalize this proof, we need the following theorem.
\begin{theorem}[Theorem $4$ in~\citet{hille1958roots}]
    Let $A, B \in \CC^{n \times n}$. If there are no two eigenvalues in $A$ such that their difference is of the form $2n\pi i$ for $n > 0$ and, if $e^A = e^B$, we have that $AB = BA$.
\end{theorem}

With this theorem in hand we can prove the following strengthened result.
\begin{theorem}[Properties of the Lie exponential]
    Let $G$ be a closed subgroup of $\GL{n, \CC}$, the Lie exponential is a diffeomorphism on the set $U = \set{A \in \g | \abs{\Im\pa{\lambda_i(A)}} < \pi}$
    with $\lambda_i(A)$ the eigenvalues of $A$.
\end{theorem}
\begin{proof}
    The fact that the differential of the exponential is surjective on this domain is classic (\cf Section $1$, Proposition $7$ in~\citet{rossmann2006lie}). As such, we just have to prove that the exponential is injective on this domain.

    If $A \in U$ is diagonalizable, $A = C\Sigma C^{-1}$ with $\Sigma$ diagonal, and $\exp(A) = C\exp(\Sigma)C^{-1}$ where $\exp(\Sigma)$ is just the element-wise exponential of the diagonal elements.

    By~\citet{hille1958roots}, for any two matrices $A, B \in U$, if $e^A = e^B$ we have that $AB = BA$.
    In particular, as they commute, we have that $e^A e^{-B} = e^{A - B} = \I$.

    As $A, B \in U$, we have that $\Im\pa{\lambda_i\pa{A - B}} < 2 \pi$, and as the eigenvalues of $e^{A-B}$ are $1$, and the eigenvalues of the exponential of a matrix is the exponential of its eigenvalues, we have that the eigenvalues of $A-B$ are all zero. Putting it in Jordan-normal form, we can assume that $A-B$ is upper triangular so, as the eigenvalues of $A-B$ are zero, we can assume that $A-B$ is also nilpotent.

    Now, if we prove that the only upper triangular nilpotent matrix that is mapped to the identity matrix under the exponential is the null matrix, we finish the proof, as this would imply that $A = B$.

    The set of upper-triangular nilpotent matrices is the Lie algebra of the Lie group of upper triangular matrices with ones on the diagonal. Recall the formula for the logarithm
    \[
        \log(B) = \sum_{k=1}^\infty (-1)^{k+1}\frac{\pa{B-\I}^k}{k}.
    \]
    Whenever $B$ is upper triangular with ones on the diagonal, $B-\I$ is nilpotent, so the series converges. As such, all these matrices have one and just one logarithm in $U$. In particular, the exponential is a bijection on this set.
\end{proof}

\section{Gradient of the Matrix Exponential}\label{sec:matrix_grad}
In this section we give a formula for the gradient of the pullback of a function by the matrix exponential. The implementation of these formulas in practice and how can they be applied on different manifolds is considered in~\Cref{sec:examples}.

We will prove a stronger result, which also applies to other matrix functions like $\cos(X)$, $\sin(X)$ and, with minor modifications, to functions like $\sqrt{X}$, $X^{1/n}$, and $\log(X)$.\footnote{See the remark after the proof of the theorem.}

\begin{theorem}\label{thm:analytic}
    Consider a real analytic function
    \[
        \deffun{\phi : \RR -> \RR; x -> \sum_{n=0}^\infty \frac{a_n}{n!}x^n}
    \]
    with associated matrix function
    \[
        \deffun{\phi : \M{n} -> \M{n}; X -> \sum_{n=0}^\infty \frac{a_n}{n!}X^n}
    \]

    We then have that, for the canonical inner product $\pa{A_1, A_2} = \tr\pa{\trans{A_1}A_2}$,
    \[
        \pa{\dif \phi}^\ast_X = \pa{\dif \phi}_{\trans{X}} \qquad X \in \M{n}.
    \]
\end{theorem}
\begin{proof}
    We can compute the differential of $\phi$ as
    \[
        \pa{\dif \phi}_X(E) = \sum_{n=0}^\infty \pa[\Big]{\frac{a_n}{n!} \sum_{i=0}^n X^iEX^{n-i}}.
    \]
    By linearity, it is enough to compute the adjoint of functions of the form $X \mapsto X^iEX^{n-i}$.

    Observe that the adjoint of the left multiplication $L_A(X) = AX$ is exactly $L_{\trans{A}}$
    \[
        \scalar{L_A(X), Y} \defi \tr\pa{\trans{\pa{AX}}Y} = \tr\pa{\trans{X}\trans{A}Y} = \scalar{X, L_{\trans{A}}(Y)}.
    \]
    In the case of right multiplication, we also get $R^\ast_A = R_{\trans{A}}$.

    Finally, we just have to apply this formula to the functions $L_{X^i}(E) = X^iE$ and $R_{X^{n-i}}(E) = EX^{n-i}$, and noting that $X \mapsto X^i E X^{n-i} = L_{X^i}(R_{X^{n-i}}(E))$, and that for any two functions, $(f \circ g)^\ast = g^\ast \circ f^\ast$, we get the result.
\end{proof}

After obtaining this more general result, we thought that this should be folklore in some areas of functional analysis and numerical analysis. In fact, this result can be found without proof in~\citet[p.66]{higham2008functions}.

\begin{remark}
    The generalization of this result to functions complex functions is direct, modulo computing the differential of the analytic function with conjugate coefficients in its Taylor series. In this case, one can interpret this theorem by saying that ``the adjoint of the differential is the differential of the adjoint at the adjoint'', noting the two different meanings of the word \emph{adjoint} in the sentence.

    In the complex setting, one can formulate the theorem for a holomorphic function defined just on an open subset $U \subset \CC$, and define the function on matrices on the set of matrices such that their spectrum is contained in $U$, hence making sense also of functions like $\log(X)$.

    The result still holds true for many other inner product in $\CC^{n \times n}$ (or $\M{n}$), in particular, for those for which for every matrix $X$ there exists a matrix $Y$ such that $L^\ast_X = L_Y$. If this is the case, we write $X^\ast \defi Y$ and the theorem still holds true, as in this case, $R^\ast_X = R_{X^\ast}$. Most of the scalar products on matrix spaces that appear in differential geometry have this property. For example, if we have a symmetric positive definite matrix $G \in \M{n}$ and we define the following product $\scalar{X, Y} \defi \tr\pa{\trans{X}GY}$, then we have that $X^\ast = \trans{\pa{GXG^{-1}}}$.
\end{remark}

We can now state the case of $\exp(X)$ as a corollary of~\Cref{thm:analytic} and~\Cref{corol:3_grad_abstract}.
\begin{corollary}[Gradient of the matrix parametrization]\label{corol:grad_exp_proof}
    Let $\deffun{f : \GL{n} -> \RR;}$ be a smooth function, the gradient of $f \circ \exp$ at a matrix $A \in \gl{n} \iso \M{n}$ with respect to the canonical metric at a matrix $B \in\GL{n}$, $\scalar{A_1, A_2}_B = \tr\pa{\trans{A_1}A_2}$ is given by
    \[
        \grad \pa{f \circ \exp}\pa{A} = \pa{\dif \exp}_{\trans{A}}\pa{\grad f(e^A)}.
    \]
\end{corollary}

Using the chain rule, we can also compute the gradient with respect to the dynamic Lie trivialization $\exp_B$.
\begin{corollary}\label{corol:gl}
    Let $\deffun{f : \GL{n} -> \RR;}$ be a smooth function, and let $B \in \GL{n}$. The gradient of $f \circ \exp_B$ at a matrix $A \in T_B\GL{n} \iso \gl{n} \iso \M{n}$ with respect to the canonical metric $\scalar{A_1, A_2}_B = \tr\pa{\trans{A_1}A_2}$ is given by
    \[
        \grad \pa{f \circ\exp_B}\pa{A} = \trans{\pa{B^{-1}}} \pa{\dif \exp}_{\trans{\pa{B^{-1}A}}}\pa{
        \trans{B}\grad f(\exp_B(A))}.
    \]
\end{corollary}

\begin{remark}
    These two corollaries still hold if we replace $\GL{n}$ by any real matrix Lie group with this metric. The complex case is analogous.
\end{remark}

\begin{remark}
    In~\cite{lezcano2019cheap} the following slightly different formula for the gradient of the exponential is derived for compact real matrix Lie groups:
    \[
        \grad\pa{f \circ \exp}(A) = e^A\pa{\dif \exp}_{-A}\pa{e^{-A}\grad f(\exp(A))}.
    \]
    This formula agrees with the one presented here, as it turns out that multiplication by $e^A$ commutes with $\pa{\dif \exp}_{-A}$. This can be seen, for example, modifying the proof of formula for the derivative of exponential map in~\citet[Chapter 1, Theorem 5]{rossmann2006lie} to obtain
    \[
        \pa{\dif \exp}_A(X) = \sum_{k=0}^\infty \frac{\pa{-\ad_A}^k}{(k+1)!}(e^A X).
    \]
    Finally, if $G$ is a real compact matrix Lie group together with a bi-invariant metric, one has that for every $A \in \g$, $A^\ast = -A$, where $A^\ast$ should be understood in the sense of $L^\ast_{A} = L_{A^\ast}$. This can be seen, for example, considering that a real compact matrix Lie group is either a subgroup of the orthogonal group or a conjugate of one. Using this, we finally see that the formula presented in~\cite{lezcano2019cheap} is equivalent to~\Cref{corol:grad_exp_proof}.
\end{remark}

\section{Examples of Matrix Manifolds and Specific Trivializations}\label{sec:examples}
This section has an expository purpose. It is intended as a compilation of useful results for the implementation of different trivializations. We will go over the forms that the Lie exponential and the Riemannian exponential---geodesics---take in different manifolds that are useful in the field of machine learning.

We will deliberately develop as least theory as possible, but we will still point out the relevant literature sources as \emph{remarks}, for those interested in the theoretical background. At the end, we will also describe some retractions, which are useful for problems on which computing the geodesics or the Lie exponential is too expensive.

We will put as examples some Lie groups, the sphere and the hyperbolic space, the Stiefel manifold, and the space of symmetric positive definite matrices.

\begin{remark}
    On some of the manifolds considered below, the metric is not the canonical one given by $\scalar{A_1, A_2}_B = \tr\pa{\trans{A_1}A_2}$, but often a left-translation of this one of the form
    \[
    \scalar{A_1, A_2}_B = \tr\pa{\trans{\pa{B^{-1}A_1}}B^{-1}A_2} \qquad \forall A_1, A_2 \in T_B\MM.
    \]
For these metrics, when we compute the gradient, we cannot use~\Cref{corol:grad_exp_proof} directly. On the other hand, after a similar reasoning, we get that the differential with respect to these metrics is given by the formula
\[
    \grad \pa{f \circ \exp}\pa{A} = B\pa{\dif \exp}_{\trans{A}}\pa{B^{-1} \grad f(e^A)}.
\]
We can also deduce this formula just noting that, for these metrics, left translations are isometries by construction.
\end{remark}

\subsection{Compact matrix Lie groups}\label{sec:compact}
On a Lie group, we can identify all the tangent spaces using left multiplication. In particular, we have that
\[
    T_BG = \set{ BA | A \in \g} \qquad B \in G
\]
where $\g \defi T_e G$ is the tangent space at the identity: the Lie algebra of $G$. As such, if we know the structure of $\g$ we can parametrize any tangent space of $G$.

For compact Lie groups, the Lie exponential and the Riemannian exponential agree\footnote{Here we are assuming that we consider the group $G$ together with a bi-invariant metric. For compact matrix Lie groups this metric is exactly $\scalar{A_1, A_2}_B = \tr\pa{\trans{A_1}A_2}$. For more on this, we refer the reader~\citet[][Appendix C.1.]{lezcano2019cheap}.} and take the form
\begin{equation}\label{eq:geodesics_compact}
    \exp_{G, B}(\tilde{A}) = \exp_B(\tilde{A}) = B\exp(B^{-1}\tilde{A}) = B \exp(A) \qquad A \in \g.
\end{equation}
where $\exp$ is the exponential of matrices and we still used the identification $\tilde{A} = BA$.
For these groups, the Riemannian exponential is surjective.

These Lie groups were already presented in~\citet{lezcano2019cheap} in the context of optimization for neural networks.  In that paper, this trivialization was only considered in the static case, namely $\deffun{\exp : \g -> G;}$.

The gradient of~\Cref{eq:geodesics_compact} is given by \Cref{prop:grad_exp}.

Finally, for compact matrix Lie groups, in order to use this formula to implement the dynamic trivialization method, we are just missing the expression for the Lie algebra $\g \subset \M{n}$ of the Lie group in which we are interested. We give a list of some of these below.
\paragraph{Special orthogonal group}
\[
    \SO{n} = \set{B \in \M{n} | \trans{B}B = \I, \det B = 1} \qquad
    \so{n} = \Skew{n} = \set{A \in \M{n} | \trans{A} = -A},
\]
\paragraph{Unitary group}
\[
    \U{n} = \set{B \in \CC^{n \times n} | \conj{B}B = \I} \qquad
    \ualg{n} = \set{A \in \CC^{n \times n} | \conj{A} = -A},
\]
\paragraph{Special unitary group}
\[
    \SU{n} = \set{B \in \CC^{n \times n} | \conj{B}B = \I, \det B = 1} \qquad
    \su{n} = \set{A \in \CC^{n \times n} | \conj{A} = -A, \tr A = 0}.
\]
\paragraph{Complex torus}
\[
    \TT{n, \CC} = \set{B \in \Diag{n, \CC} | \abs{B_{ii}} = 1} \qquad
    \ttalg{n, \CC} = \set{A \in \Diag{n, \CC} | A_{ii} \in i\RR \subset \CC},
\]

\begin{remark}
    We say that $\TT{n, \CC}$ is a torus because it is a product of $n$ circles. This can easily be seen simply defining the circle as $S^1 = \set{z \in \CC | \abs{z} = 1}$, so that $\TT{n, \CC} \iso S^1 \times \dots \times S^1$. In this case, the correspondence between the Lie algebra and the Lie group is given by the Euler formula.
\end{remark}

\paragraph{Real torus}
    The real torus $\TT{2n, \RR}$ consists of the $2n \times 2n$ block-diagonal matrices with blocks of the form
    \[
\begin{pmatrix}
    \cos(\theta) &-\sin(\theta) \\
    \sin(\theta) &\cos(\theta)
\end{pmatrix} \qquad \theta \in [-\pi, \pi].
    \]
    Note that is exactly the matrix representation of the complex number $e^{\theta i}$. In this case, its Lie algebra is given by the block-diagonal matrices with blocks given by
    \[
\begin{pmatrix}
    0 &-a \\
    a &0
\end{pmatrix} \qquad a \in \RR.
    \]

    \begin{remark}
        In the case of the real and complex torus, the exponential is a Riemannian covering map, this meaning that, in particular, it is always a local isometry, and it does not create local minima or saddle points. For this reason, to optimize on these two manifolds, we would not need to use dynamic trivializations, given that a static trivialization would work just fine as a direct corollary of~\Cref{thm:change_metric}.
    \end{remark}

\subsection{The groups \texorpdfstring{$\GLp{n}$}{GL+(n)} and \texorpdfstring{$\SL{n}$}{SL(n)}}
We then look at two more Lie groups on which we can compute the Riemannian exponential. These groups are of primal importance for problems that require invertible matrices or the study of volume-flows, like normalizing flows.

These groups are also an important example of groups on which the Riemannian exponential and the Lie exponential do not agree, and thus, in this case, we have two different trivialization schemes.

Furthermore, the Lie exponential is not surjective on these groups, so these are also examples of a retraction that could not be used as a static trivialization, but it can be used as a dynamic one.

\paragraph{Positive general linear group}
\[
    \GLp{n} = \set{B \in \M{n} | \det B > 0 } \qquad
    \gl{n} = \M{n}.
\]
This is the connected component containing the identity matrix of the general linear group
\[
    \GL{n} = \set{B \in \M{n} | \det B \neq 0 }.
\]

\paragraph{Special linear group}
\[
    \SL{n} = \set{B \in \M{n} | \det B = 1} \qquad
    \slalg{n} = \set{A \in \M{n} | \tr A = 0}.
\]

\begin{remark}
The orthogonal projection from $\RR^{n \times n}$ onto $\slalg{n}$ is given by
\[
    \deffun{\pi_{\slalg{n}} : \RR^{n \times n} -> \slalg{n}; A -> A - \tfrac{1}{n}\tr(A)\I}
\]
We can use this formula to parametrize $\slalg{n}$, in the same way that we use $A \mapsto \frac{1}{2}\pa{A - \trans{A}}$ to parametrize $\so{n} \iso \Skew{n}$.
\end{remark}

On these groups have two different trivializations based on the exponential of matrices.

On the one hand, we still have the dynamic Lie trivialization $\exp_B$ presented in~\Cref{sec:compact}.

On the other hand, if $G$ is $\GL{n}$ or $\SL{n}$ for $n > 2$ equipped with the metric $\scalar{\tilde{A}_1, \tilde{A}_2}_B = \tr\pa{\trans{\pa{B^{-1}\tilde{A}_1}}B^{-1}\tilde{A}_2}$, we have that the Riemannian trivialization for these groups is given by
\footnote{This result applies not only to $\SL{n}$, but to any semisimple Lie group equipped with the left-invariant metric associated to the Killing form.}
\[
    \exp_{G, B}(BA) = B\exp\pa{\trans{A}}\exp\pa{A - \trans{A}} \quad \text{for } A \in \g.
\]
Note that $BA \in T_BG$, as one would expect.

\begin{remark}
This result was first stated in~\citet{wang1969discrete}, and a proof of it can be found in~\citet[][Chapter 6, Exercise A.9]{helgason1979differential}. For the proof for $\GL{n}$, see~\citet[][Theorem $2.14$]{andruchow2014left}.
\end{remark}

We can then compute the gradient of this parametrization as we know how to compute the gradient of the exponential map with respect to this metric, as detailed at the beginning of~\Cref{sec:examples}.

\begin{remark}
    It happens that the Lie exponential is not surjective on $\SL{n}$ so, in this case, it would not be possible to set $K=\infty$ in the dynamic trivialization algorithm, that is, it would be necessary to change the basis of the trivialization. The Lie trivialization is not surjective on $\GLp{n, \RR}$ either, but it is surjective on $\GL{n, \CC}$, with $\gl{n, \CC} \iso \CC^{n \times n}$.

    These are examples for which using dynamic trivializations allow us to use certain parametrizations that we would not be able to use in the context of static trivializations.

    The Riemannian exponential on $\SL{n}$ and $\GL{n, \RR}$ is surjective with this metric.
\end{remark}

\begin{remark}
On these two manifolds, we can also use their polar decomposition as a trivialization to optimize over them, see~\citet[][Proposition $2.19$]{hall2015lie}.
\end{remark}

\subsection{Naturally reductive homogeneous spaces}
In this section we touch on a few of the most used manifolds in optimization, namely the Stiefel manifold, the sphere, the hyperbolic space, and the symmetric positive definite matrices.

In this section we will restrict ourselves to expose the formulae for the exponential on these manifolds for certain metric. Most of these manifolds fall under the theory of symmetric manifolds, or the more general theory of naturally reductive homogeneous spaces. For a derivation of the fomulae in this section in the more general context of naturally reductive homogeneous spaces, we refer the reader to the self-contained exposition in~\citet[][Chapter 22]{gallier2019differential}.

\subsubsection{Stiefel manifold}
The Stiefel manifold is the manifold of $n \times k$ matrices with $k \leq n$ with orthonormal columns. Equivalently, it is the set of orthonormal $k$-frames on $\RR^n$. In symbols we can see the Stiefel manifold as a submanifold of $\M{n, k}$ as follows:
\[
    \St{n, k} \defi \set{B \in \M{n, k} | \trans{B}B = \I_k} \qquad T_B\St{n, k} = \set{\tilde{A} \in \M{n,k} | \trans{B}\tilde{A} \in \so{k}}
\]
Note that $\St{n, n} \iso \Ort{n}$. In this case, compare the formula of the tangent space with that given for $T_B\SO{n}$ Lie groups in~\Cref{sec:compact}, in particular that of $\so{n}$.

If we consider any completion of the frame $B$ into a basis of $\RR^n$, that is, a matrix $B_{\bot} \in \M{n, n-k}$ such that
$
\begingroup
\setlength\arraycolsep{2pt}
\begin{pmatrix}
B & B_{\bot}
\end{pmatrix}
\endgroup
\in \Ort{n}$,
we have the more computationally amenable description of the tangent spaces of $\St{n,k}$
\[
    T_B\St{n, k} = \set{BA + B_\bot A_\bot \in \M{n,k} | A \in \so{k}, A_\bot \in \M{n-k, k}}.
\]
Note that if $n = k$, $T_B\St{n,n} = \set{BA | A \in \so{n}}$ and we still recover the same definition from~\Cref{sec:compact}.

The canonical metric
\footnote{We say that this is the canonical metric because it is the one inherited from the quotient structure---as a homogeneous space---of $\St{n, k}$ as $\St{n,k} \iso \Ort{n} / \Ort{n-k}$. If we put the Euclidean metric $\tr\pa{\trans{X}Y}$ on $\Ort{n}$, this metric is bi-invariant under the action of $\Ort{n-k}$ and descends into the canonical metric on the quotient manifold $\Ort{n} / \Ort{n-k}$ described here. For the exact computations see~\citet{edelman1998geometry}.}
on the Stiefel manifold is given for $B \in \St{n, k}$, $\tilde{A}_1, \tilde{A}_2 \in T_B\St{n, k}$ by
\[
    \scalar{\tilde{A}_1, \tilde{A}_2}_B = \tr\pa{\tilde{A}^{\transaux}_1\pa{\I_n - \tfrac{1}{2}B\trans{B}}\tilde{A}_2}
\]

With the notation as above, consider the QR decomposition $QR = (\I_n - B\trans{B})\tilde{A}$ with $Q \in \St{n, k}$, $R \in \M{k}$, then we have that is the we have that the Riemannian exponential is given
\[
    \exp_{\St{n, k}, B}(\tilde{A}) =
    \begin{pmatrix}
        B & Q
    \end{pmatrix}
    \exp
    \begin{pmatrix}
        A &-R \\
        R &0
    \end{pmatrix}
    \begin{pmatrix}
        I_k \\
        0
    \end{pmatrix}.
\]
\begin{remark}
    The computational cost of computing geodesics on $\St{n, k}$ is then dominated by the computation of a thin-QR factorization of a $n \times k$ matrix and the computation of a exponential of a skew-symmetric $2k \times 2k$ matrix.

    If $2k > n$, a more efficient algorithm is possible. We just have to compute the geodesics on $\SO{n}$ as per~\Cref{sec:compact} and then drop then project the result onto $\St{n, k}$ dropping the last $n-k$ columns. This process requires the computation of just one exponential of an $n \times n$ matrix. This process is equivalent to the formula described above.
\end{remark}

\begin{remark}
In~\citet[][Section 2.2.2]{edelman1998geometry} the authors give a formula for the geodesics of $\St{n, k}$ seen as a submanifold of $\M{n, k}$, that is, with the metric $\scalar{\tilde{A}_1, \tilde{A}_2} = \tr\pa{\tilde{A}^\transaux_1\tilde{A}_2}$. In Section 2.4.1 they also discuss an essential difference between the Euclidean metric and the canonical metric on the Stiefel manifold.
\end{remark}

\subsubsection{The sphere and the hyperbolic plane}
The case of the sphere $S^n = \set{x \in \RR^{n+1} | \norm{x} = 1}$ is probably one of the most classical ones. We will always consider the \emph{round sphere}, this is, the sphere as a subset of $\RR^{n+1}$ together with the metric inherited from $\RR^{n+1}$.

Its tangent space at a point $x \in S^n$ is simply given by the set of vectors orthogonal to it
\[
    T_xS^n = \set{v \in \RR^n | \scalar{x,v} = 0}.
\]
and the geodesics are given by
\[
    \exp_{S^n, x}(v) = \cos(\norm{v})x + \sin(\norm{v})\frac{v}{\norm{v}}.
\]

To describe the $n$-dimensional hyperbolic space, first consider the diagonal matrix $\I_{n, 1}$
with $n$ positive ones and a negative one in its diagonal. We will use the following notation
\[
    \scalar{x, y}_{\Haux} \defi \scalar{x, \I_{n,1}y} = \sum_{i=1}^n x_iy_i - x_{n+1}y_{n+1} \qquad \forall x, y \in \RR^{n+1}
\]
and denote by $\norm{x}_{\Haux} = \sqrt{\scalar{x, x}_{\Haux}}$ whenever $\scalar{x,x}_{\Haux} \geq 0$.

With this notation, the $n$-dimensional hyperbolic space $\Hip{n}$ can be seen as the submanifold of $\RR^{n+1}$ defined by
\[
    \Hip{n} = \set{ x \in \RR^{n+1} | \scalar{x,x}_{\Haux} = -1, x_{n+1} > 0}
\]
with tangent space at $x \in \Hip{n}$ given by
\[
    T_x\Hip{n} = \set{ v \in \RR^{n+1} | \scalar{x, v}_{\Haux} = 0}.
\]

The geodesics are then given by
\[
    \exp_{\Hip{n}, x}(v) = \cosh(\norm{v}_{\Haux})x + \sinh(\norm{v}_{\Haux})\frac{v}{\norm{v}_{\Haux}}.
\]

\begin{remark}
    The formula for the sphere is just a particular case of the one given for $\St{n+1, 1} \iso S^n$.

    The reason why the formulas of the geodesics on the sphere and the hyperbolic plane are so similar has a geometric meaning. This can be seen in a more general case, considering the oriented Grassmannian manifold and the hyperbolic Grassmannian. The sphere and the hyperbolic plane are special cases of these manifolds. These manifolds are \emph{symmetric spaces} and they are dual to each other. For more on the duality between symmetric spaces of compact and non-compact type, we refer the reader to~\citet[][Chapter 5, Example 1]{helgason1979differential} or~\citet[][Chapter 11]{o1983semi}.
\end{remark}

\begin{remark}
    In the same spirit as we can compute the geodesics on $\St{n,k}$ by taking a geodesic in $\SO{n}$ and projecting it down to $\St{n,k}$, we can also compute the geodesics of the real projective plane $\RP{n}$ by computing the geodesic on $S^n$ and projecting it down to $\RP{n}$. The metric induced on $\RP{n}$ is called the \emph{standard round metric} on $\RP{n}$. If we perform the same process between $S^{2n+1}$ and $\CP{n}$ and, in this case, we would get the \emph{Fubini-Study} metric. This construction arises naturally in the context of principal bundles with invariant metrics.
\end{remark}

\subsubsection{The symmetric positive definite matrices}\label{sec:pos_def}
The symmetric positive definite matrices $\Symp{n}$ do not form a Lie group, as they are not closed under matrix multiplication, but they are a symmetric space.

When seen as a subset of $\M{n}$, we can endow it with a left-invariant metric defined as $\scalar{\tilde{A}_1, \tilde{A}_2}_B = \tr\pa{B^{-1}A_1B^{-1}A_2}$. The tangent space at a point $B \in \Symp{n}$ is given by
\[
    T_B\Symp{n} = \set{ B^{1/2}AB^{1/2} | A \in \sym{n}}
\]
where $\sym{n}$ is the tangent space at the identity, given by the symmetric matrices
\[
    \sym{n} = \set{A \in \M{n} | \trans{A} = A}.
\]
Note that for a symmetric positive definite matrix the square root is well defined, as symmetric positive definite matrices are diagonalizable, and the square root is just the matrix whose eigenvalues are the (positive) square root of the eigenvalues of the initial matrix.

Following the notation for Lie groups, if we denote $\tilde{A} = B^{1/2}AB^{1/2}$, we have that
\[
    \exp_{\Symp{n}, B}(\tilde{A}) = B^{1/2}\exp\pa{B^{-1/2} \tilde{A} B^{-1/2}}B^{1/2} = B^{1/2}\exp\pa{A}B^{1/2}.
\]

\begin{remark}
    In this case, this manifold also constitutes an example of a symmetric space since $\Symp{n} \iso \GLp{n} / \Ort{n}$. The metric considered here is the natural one with respect to this structure. An introduction to the computational aspects of this manifold can be found in~\citet{bonnabel2009riemannian}.
\end{remark}

\subsection{Some retractions}
For now we have just mentioned examples regarding either the Lie exponential or the Riemannian exponential, but the dynamic trivialization framework allows us to use any function that is a retraction. In order to make use of arbitrary retractions, we just have to be able to compute the gradient of the function when precomposed with them. We will do so for a few important examples in this section.

In the case of the two retractions mentioned in~\Cref{sec:retractions}, the Cayley map and projectors, their derivatives are already implemented in the major deep-learning packages, like Pytorch or Tensorflow. The first one just requires an inverse (or, more efficiently and stable, the solution of a system of the form $AX = B$) and the second one just requires the derivatives with respect to the SVD decomposition.

The retraction induced by a projector can be easily implemented for most manifolds. For example, for the sphere takes just the form $x \mapsto \frac{x}{\norm{x}}$, whose derivative can also be computed just using autodiff.

For the symmetric positive definite matrices, we have the retraction from the symmetric matrices into the positive semidefinite matrices given by $A \mapsto A^2$. This one is similar to the frequently used from the upper triangular matrices given by the Cholesky decomposition $L \mapsto L\trans{L}$. The former has the advantage that we have access to $A$ which is the square root of its image. This can be helpful, as sometimes the square root of the matrix is needed for some computations, as we have seen in~\Cref{sec:pos_def}. The retraction given by the Cholesky decomposition has the advantage that, if the diagonal of the upper-triangular matrix $L$ is strictly positive, then $L\trans{L}$ will be positive definite. For this reason this retraction is often used to parametrize variance kernels in Bayesian statistics.

Another retraction for $\Symp{n}$ is given by the exponential of matrices $\deffun{\exp : \sym{n} -> \Symp{n};}$ which is a diffeomorphism. As such, it provides a rather good, although expensive, option to parametrize this manifold.

For a much more in-depth treatment of retractions, we refer the reader to~\citet{absil2009optimization}.

\section{Detailed Experiment Set-Up and Hyperparameters}\label{sec:hyperparam}
We tried to reproduce as faithfully as possible the set-up from previous experiments, to achieve a fair comparison. The batch size for all the experiments is $128$. We fixed the seed to be $5544$ of both Numpy and Pytorch for reproducibility in the final runs.

The exact architecture to process for a sequence of inputs $x_t \in \RR^d$ with a hidden size $p$ is given by the formula
\[
    h_{t+1} = \sigma\pa{\exp(A) h_t + Tx_{t+1}}
\]
with $A \in \Skew{p}$ and $T \in \M{p,d}$. $\sigma$ is the \code{modrelu} non-linearity introduced in~\citet{arjovsky2016unitary}.

The initialization \emph{Henaff} refers to initializing the diagonal blocks of the skew-symmetric matrices with elements sampled from the uniform distribution $\Unif\pa{-\pi, \pi}$ as detailed in~\citet{henaff2016recurrent}. The \emph{Cayley} initialization refers to sampling the diagonal from a distribution $u \sim \Unif\pa{0, \pi/2}$ and then computing $s = -\sqrt{\frac{1-\cos(u)}{1 + \cos(u)}}$ as detailed in~\citet{helfrich18a}.

As we mentioned in the experiments section, we did not include the copying experiment that was usually used in previous papers, given that, as it was demonstrated in~\citet{lezcano2019cheap}, the exponential trivialization converges to the exact solution even when based at the identity. The same happens when used with the dynamic trivialization, so we do not think that this experiment adds anything to the results.

\begin{table}[!ht]
\begin{adjustwidth}{-8.6pt}{-8.6pt}
    \centering
    \caption{Hyperparameters for \dtriv{}$1$.}
    \label{tab:hyperparam_1}
    \begin{tabular}{ll|cccc}
        \toprule
        Dataset & Size & Optimizer & Learning Rate & Orthogonal optimizer & Orthogonal Learning Rate \\
        \midrule
        \midrule
        \multirow{3}{*}{\mnist} &
        $170$ &
        \multirow{6}{*}{\rmsprop} &
        $10^{-3}$ &
        \multirow{6}{*}{\rmsprop} &
        $10^{-4}$ \\
        & $360$ & & $10^{-3}$ & & $10^{-4}$ \\
        & $512$ & & $5\cdot 10^{-4}$ & & $7\cdot 10^{-5}$ \\
        \multirow{3}{*}{\pmnist} &
        $170$ &
        &
        $7 \cdot 10^{-4}$ &
        &
        $2 \cdot 10^{-4}$ \\
        & $360$ & & $7\cdot 10^{-4}$ & & $7\cdot 10^{-5}$\\
        & $512$ & & $5 \cdot 10^{-4}$ & & $5 \cdot 10^{-5}$ \\
        \midrule
        \multirow{3}{*}{\timit} &
        $224$ &
        \multirow{3}{*}{\adam} &
        $10^{-3}$ &
        \multirow{3}{*}{\rmsprop} &
        $10^{-4}$ \\
        & $322$ & & $10^{-3}$ & & $10^{-4}$ \\
        & $425$ & & $10^{-3}$ & & $10^{-4}$ \\
        \bottomrule
    \end{tabular}
\end{adjustwidth}
\end{table}

\begin{table}[!ht]
\begin{adjustwidth}{-8.6pt}{-8.6pt}
    \centering
    \caption{Hyperparameters for \dtriv{}$100$.}
    \label{tab:hyperparam_100}
    \begin{tabular}{ll|cccc}
        \toprule
        Dataset & Size & Optimizer & Learning Rate & Orthogonal optimizer & Orthogonal Learning Rate \\
        \midrule
        \midrule
        \multirow{3}{*}{\mnist} &
        $170$ &
        \multirow{6}{*}{\rmsprop} &
        $5 \cdot 10^{-4}$ &
        \multirow{6}{*}{\rmsprop} &
        $10^{-4}$ \\
        & $360$ & & $3\cdot 10^{-4}$ & & $5\cdot 10^{-5}$ \\
        & $512$ & & $5 \cdot 10^{-4}$ & & $10^{-4}$ \\
        \multirow{3}{*}{\pmnist} &
        $170$ &
        &
        $7 \cdot 10^{-4}$ &
        &
        $10^{-4}$ \\
        & $360$ & & $5 \cdot 10^{-4}$ & & $7 \cdot 10^{-5}$\\
        & $512$ & & $5 \cdot 10^{-4}$ & & $5 \cdot 10^{-5}$ \\
        \midrule
        \multirow{3}{*}{\timit} &
        $224$ &
        \multirow{3}{*}{\adam} &
        $10^{-3}$ &
        \multirow{3}{*}{\rmsprop} &
        $2 \cdot 10^{-4}$ \\
        & $322$ & & $10^{-3}$ & & $2 \cdot 10^{-4}$ \\
        & $425$ & & $10^{-3}$ & & $10^{-4}$ \\
        \bottomrule
    \end{tabular}
\end{adjustwidth}
\end{table}

\begin{table}[!ht]
\begin{adjustwidth}{-8.6pt}{-8.6pt}
    \centering
    \caption{Hyperparameters for \dtriv{}$\infty$.}
    \label{tab:hyperparam_infty}
    \begin{tabular}{ll|cccc}
        \toprule
        Dataset & Size & Optimizer & Learning Rate & Orthogonal optimizer & Orthogonal Learning Rate \\
        \midrule
        \midrule
        \multirow{3}{*}{\mnist} &
        $170$ &
        \multirow{6}{*}{\rmsprop} &
        $7\cdot 10^{-4}$ &
        \multirow{6}{*}{\rmsprop} &
        $10^{-4}$ \\
        & $360$ & & $5 \cdot 10^{-4}$ & & $10^{-4}$ \\
        & $512$ & & $10^{-4}$ & & $7 \cdot 10^{-5}$ \\
        \multirow{3}{*}{\pmnist} &
        $170$ &
        &
        $7\cdot 10^{-4}$ &
        &
        $2 \cdot 10^{-4}$ \\
        & $360$ & & $7 \cdot 10^{-4}$ & & $5 \cdot 10^{-5}$\\
        & $512$ & & $3 \cdot 10^{-4}$ & & $7 \cdot 10^{-5}$ \\
        \midrule
        \multirow{3}{*}{\timit} &
        $224$ &
        \multirow{3}{*}{\adam} &
        $10^{-3}$ &
        \multirow{3}{*}{\rmsprop} &
        $2 \cdot 10^{-4}$ \\
        & $322$ & & $10^{-3}$ & & $2 \cdot 10^{-4}$ \\
        & $425$ & & $10^{-3}$ & & $2 \cdot 10^{-4}$ \\
        \bottomrule
    \end{tabular}
\end{adjustwidth}
\end{table}

\end{document}